\documentclass[12pt]{report}
\usepackage[a4paper,margin=2cm]{geometry}
\usepackage[english]{babel}
\usepackage[utf8]{inputenc}
\usepackage{amsmath}
\usepackage{amssymb}
\usepackage{amsthm}
\usepackage{mathtools}
\usepackage{bm}
\usepackage{wasysym}
\usepackage{graphicx, color}
\usepackage{subcaption}
\usepackage{tikz}
\usepackage{hyperref}
\usepackage{booktabs}
\usepackage{pifont}
\usepackage[nameinlink,noabbrev]{cleveref}
\usetikzlibrary{positioning}
\usetikzlibrary{arrows.meta}

\newcommand{\norm}[1]{\left\lVert#1\right\rVert}

\DeclareMathOperator{\Tr}{Tr}
\DeclareMathOperator{\LL}{LL}
\DeclareMathOperator{\sgn}{sgn}
\DeclareMathOperator{\IS}{IS}
\DeclareMathOperator{\Var}{Var}
\DeclareMathOperator{\DFT}{DFT}
\DeclareMathOperator{\IDFT}{IDFT}
\DeclareMathOperator{\Cov}{Cov}

\DeclareMathOperator{\KL}{KL}

\DeclareMathOperator{\EX}{\mathbb{E}}

\newtheorem{definition}{Definition}
\newtheorem{proposition}{Proposition}
\newtheorem{theorem}{Theorem} 

\begin{document}

\begin{titlepage}

    \newcommand{\HRule}{\rule{\linewidth}{0.5mm}} 

    \center 




    \textsc{\Large MSc Artificial Intelligence}\\[0.2cm]

    \textsc{\Large Master Thesis}\\[0.5cm]




    \HRule \\[0.4cm]

    { \huge \bfseries Poolformer: Recurrent Networks with Pooling for Long-Sequence Modeling}\\[0.4cm] 

    \HRule \\[0.5cm]




    by\\[0.2cm]

    \textsc{\Large Daniel Gallo Fernández}\\[0.2cm] 

    15110664\\[1cm]




    {\Large June 29, 2025}\\[1cm] 

    36 ECTS\\ %

    January 2025 - June 2025\\[1cm]%




    \begin{minipage}[t]{0.4\textwidth}

        \begin{flushleft} \large

            \emph{Supervisor:} \\

            Dr \textsc{James Townsend} 

        \end{flushleft}

    \end{minipage}

    ~

    \begin{minipage}[t]{0.4\textwidth}

        \begin{flushright} \large

            \emph{Examiner:} \\

            Dr \textsc{Jan-Willem van de Meent}\\

            \vspace{0.5cm}

            \emph{Second reader:} \\

            Dr \textsc{James Townsend}\\

        \end{flushright}

    \end{minipage}\\[2cm]




    \includegraphics[width=10cm]{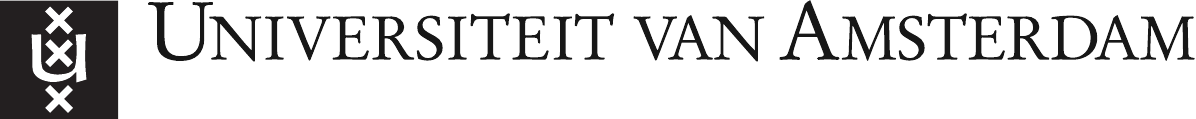}


    \vfill 

\end{titlepage}

\pagenumbering{roman}

\tableofcontents

\begin{abstract}
    Sequence-to-sequence models have become central in Artificial Intelligence, particularly following the introduction of the transformer architecture. While initially developed for Natural Language Processing, these models have demonstrated utility across domains, including Computer Vision. Such models require mechanisms to exchange information along the time dimension, typically using recurrent or self-attention layers. However, self-attention scales quadratically with sequence length, limiting its practicality for very long sequences.

    We introduce Poolformer, a sequence-to-sequence model that replaces self-attention with recurrent layers and incorporates pooling operations to reduce sequence length. Poolformer is defined recursively using SkipBlocks, which contain residual blocks, a down-pooling layer, a nested SkipBlock, an up-pooling layer, and additional residual blocks. We conduct extensive experiments to support our architectural choices.

    Our results show that pooling greatly accelerates training, improves perceptual metrics (FID and IS), and prevents overfitting. Our experiments also suggest that long-range dependencies are handled by deep layers, while shallow layers take care of short-term features.

    Evaluated on raw audio, which naturally features long sequence lengths, Poolformer outperforms state-of-the-art models such as SaShiMi \cite{SaShiMi} and Mamba \cite{mamba}. Future directions include applications to text and vision, as well as multi-modal scenarios, where a Poolformer-based LLM could effectively process dense representations of images and videos.
\end{abstract}

\pagenumbering{arabic}

\chapter{Introduction}
Sequence-to-sequence models have become central in Artificial Intelligence, particularly following the introduction of the transformer architecture. While initially developed for Natural Language Processing, these models have demonstrated utility across domains. A notable example is Computer Vision, where vision transformers are increasingly displacing traditional convolutional neural networks.

Sequence-to-sequence models require mechanisms to exchange information along the time dimension, typically using recurrent or self-attention layers. Recurrent layers need to compress the past into a fixed-size state, while self-attention uses a state that grows linearly with the sequence length. Denoting the sequence length as $S$, the time complexity of recurrent layers is $O(S)$, but this cannot be parallelized easily. Self-attention, on the other hand, is $O(S^2)$ but it is easily parallelizable. The quadratic complexity of self-attention becomes a problem when the sequence length is significantly larger than the feature dimension, which is the case that we focus on. We propose Poolformer, a sequence-to-sequence model that replaces self-attention layers with recurrent layers, and that makes use of pooling to reduce the sequence length. We evaluate our model on raw audio, which naturally features long sequence lengths.

Chapter 2 contains all the theoretical background. It starts by reviewing signal processing concepts, such as $\mu$-law-encoding (used to quantize audio before feeding it to a model), the fast Fourier transform (used to perform convolutions in linearithmic time), and mel-spectrograms (used to compute the FID and IS metrics). It continues explaining metrics used to assess the quality and diversity of generative models. We then review different kinds of sequence-to-sequence models: recurrent neural networks, self-attention, and state-space models. Finally, we explain how to parallelize recurrent neural networks using convolutions and associative scans.

Chapter 3 presents our proposed Poolformer model, that is defined recursively using SkipBlocks, which contain residual blocks, a down-pooling layer, a nested SkipBlock, an up-pooling layer, and additional residual blocks. We also explain different design decisions such as skip-connection placement, pooling and gating.

Chapter 4 presents our experiments on raw audio data. We conduct extensive experiments to support our architectural choices from Chapter 3. Our results show that pooling greatly accelerates training, improves perceptual metrics (FID and IS), and prevents overfitting. They also suggest that long-range dependencies are handled by deep layers, while shallow layers take care of short-term features. In addition, Poolformer outperforms state-of-the-art models such as SaShiMi \cite{SaShiMi} and Mamba \cite{mamba}.

Chapter 5 contains the conclusions, and future research directions, that include applications to text and vision, as well as multi-modal scenarios, where a Poolformer-based LLM could effectively process dense representations of images and videos.

\chapter{Theoretical background}
This chapter starts by reviewing signal processing concepts, such as $\mu$-law-encoding (used to quantize audio before feeding it to a model), the fast Fourier transform (used to perform convolutions in linearithmic time), and mel-spectrograms (used to compute the FID and IS metrics). It continues explaining metrics used to assess the quality and diversity of generative models. We then review different kinds of sequence-to-sequence models: recurrent neural networks, self-attention, and state-space models. Finally, we explain how to parallelize recurrent neural networks using convolutions and associative scans.
\section{Signal Processing}

\subsection{Mu-law encoding}
Most WAV files contain uncompressed audio with 16 bits per sample. If we feed this raw into a sequence-to-sequence model, we would need a vocabulary size of $2^{16} = 65 536$. Even though this is not massive for today's standards, it makes experiments slower. In addition, most of the related literature \cite{SaShiMi, mamba} uses a vocabulary size of $2^{8}$. For this reason, we use 8-bit samples, and defer 16-bit samples for future work.

To quantize the samples, we use $\mu$-law encoding, that expands small amplitudes, and squashes large ones. This makes sense because most natural audio signals have a higher density of information at lower amplitudes, while large amplitudes are less frequent. By applying $\mu$-law encoding, we allocate more quantization levels to small amplitudes (where human perception is more sensitive), and fewer levels to large amplitudes.

\begin{definition}[$\mu$-law encoding]
    \begin{equation}
        F(x)      = \sgn(x) \dfrac{\ln(1 + \mu |x|)}{\ln(1 + \mu)}, \quad  -1 \leq x \leq 1.
    \end{equation}
\end{definition}

\begin{definition}[$\mu$-law decoding]
    \begin{equation}
        F^{-1}(y)  = \sgn(y) \dfrac{(1 + \mu)^{|y|} - 1}{\mu}, \quad  -1 \leq y \leq 1.
    \end{equation}
\end{definition}

We first map a 16-bit integer to the continuous interval $[-1, 1]$, then apply $F$, and finally quantize the result to an 8-bit integer. We can undo this process by using $F^{-1}$, though there will be some information loss due to quantization. This is visualized in Figure \ref{fig:mu-law}.

\begin{figure}
    \centering
    \begin{subfigure}{0.49\textwidth}
        \centering
        \includegraphics[width=\textwidth]{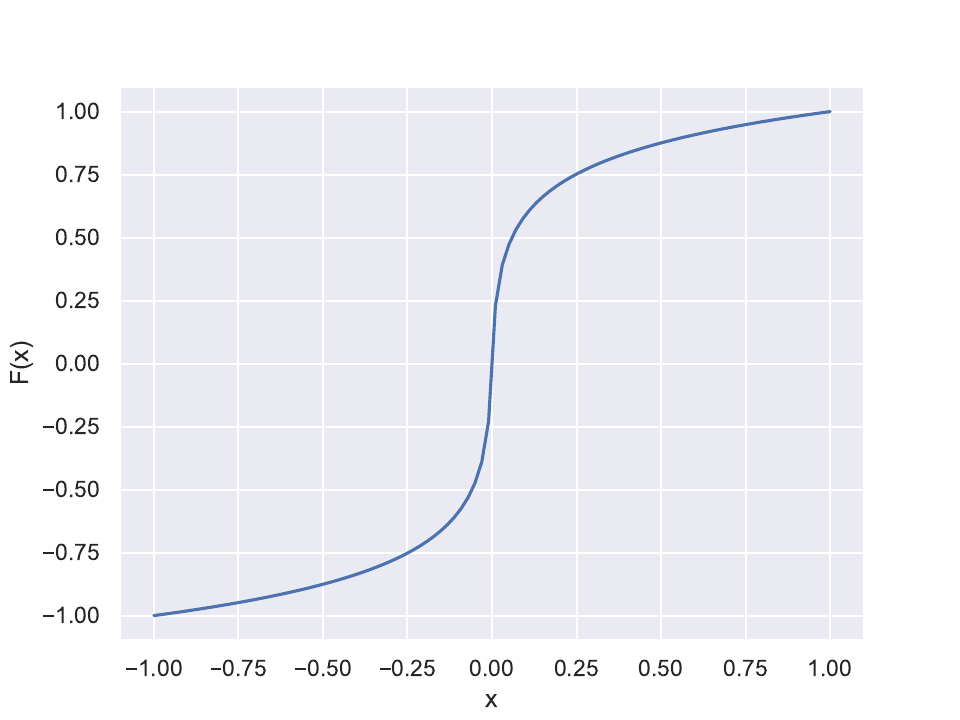}
        \caption{$\mu$-law encoding function}
    \end{subfigure}
    \hfill
    \begin{subfigure}{0.49\textwidth}
        \centering
        \includegraphics[width=\textwidth]{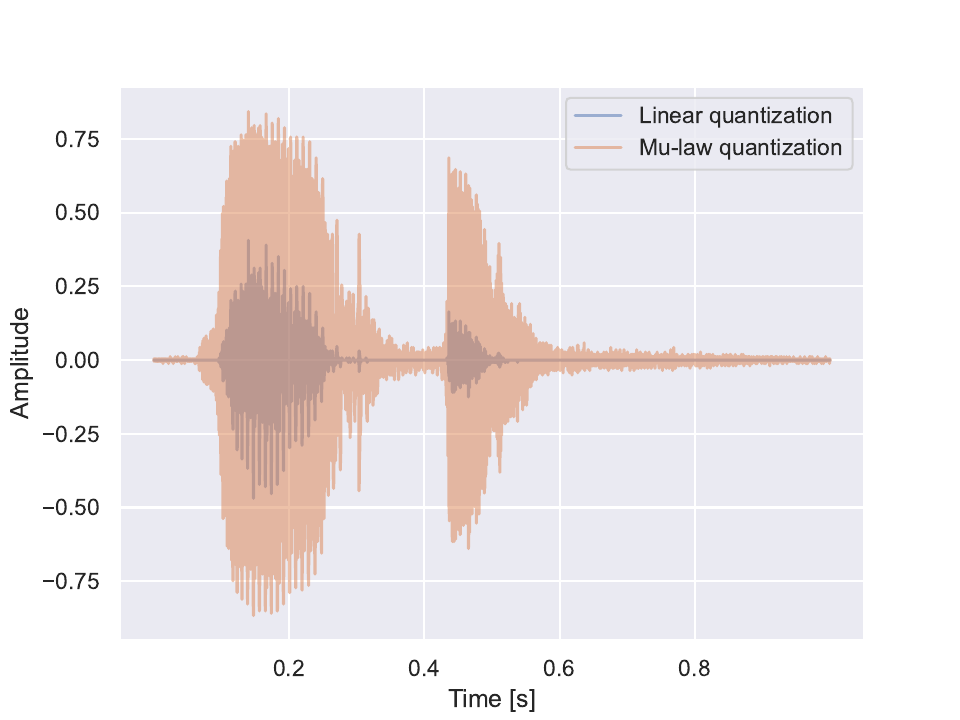}
        \caption{Comparison of different quantization schemes}
    \end{subfigure}
    \caption{\textbf{$\mu$-law encoding visualization}. The $\mu$-law function (left) compresses large amplitudes and expands small ones. When applied to a signal (right), it produces a ``stretched'' waveform.}
    \label{fig:mu-law}
\end{figure}

\subsection{Fourier transform}

It is often useful to transform raw audio waveforms to the frequency domain, as many audio properties are more easily analyzed in this representation. Our goal will be to transform audio signals (one-dimensional sequences) to mel-spectrograms, which have two dimensions: time and frequency.

\begin{definition}[Fourier transform]
    Let $f \colon \mathbb{R} \to \mathbb{C}$ be Lebesgue integrable, then
    \begin{equation}
        \hat{f}(\omega) = \int_{-\infty}^{\infty} f(x) \exp(-2 \pi i \omega x) \, dx.
    \end{equation}
\end{definition}

\begin{definition}[Inverse Fourier transform]
    Let $\hat{f} \colon \mathbb{R} \to \mathbb{C}$ be Lebesgue integrable, then
    \begin{equation}
        f(x) = \int_{-\infty}^{\infty} \hat{f}(\omega) \exp(2 \pi i \omega x) \, d\omega.
    \end{equation}
\end{definition}

The Fourier transform reveals the amplitude and phase of each frequency component present in the signal. While the Fourier transform is a powerful theoretical tool, practical applications, especially in digital signal processing, deal with discrete-time signals and require computationally feasible algorithms. This leads to the discrete Fourier transform (DFT) and its efficient implementation, the fast Fourier transform (FFT).

\subsection{Discrete Fourier Transform}
\begin{definition}[Discrete Fourier transform]
    Given $f \in \mathbb{C}^n$,
    \begin{equation}
        \DFT(f)_k = \hat{f}_k = \sum_{j = 0}^{n - 1} f_j \exp\left(-i k j \frac{2 \pi}{n}\right).
    \end{equation}
\end{definition}

\begin{definition}[Discrete inverse Fourier transform]
    Given $\hat{f} \in \mathbb{C}^n$,
    \begin{equation}
        \IDFT(\hat{f})_k = \frac{1}{n} \sum_{j = 0}^{n - 1} \hat{f}_j \exp\left(i k j \frac{2 \pi}{n}\right).
    \end{equation}
\end{definition}

\begin{theorem}
    \label{thm:inverse_fourier}
    The inverse discrete Fourier transform can be expressed in terms of the discrete Fourier transform.
    \begin{align}
        \IDFT\left(\hat{f}\right) = \frac{1}{n} \DFT\left(\hat{f}^*\right)^*.
    \end{align}
\end{theorem}
\begin{proof}
    \begin{align}
        \left(\frac{1}{n} \text{DFT}(\hat{f}^*)^* \right)_k & = \frac{1}{n} \left(\sum_{j = 0}^{n - 1} \hat{f}_j^* \exp\left(-i k j \frac{2 \pi}{n}\right) \right)^* \\
                                                            & =  \frac{1}{n} \sum_{j = 0}^{n - 1} \hat{f}_j \exp\left(i k j \frac{2 \pi}{n}\right)                   \\
                                                            & = \IDFT\left(\hat{f}\right)_k.
    \end{align}
\end{proof}
Note that the discrete Fourier transform is a linear operator, so the operation can be expressed as
\begin{align}
    \label{eq:dft}
    \hat{f} & = F_n f \notag \\
    \begin{bmatrix}
        \hat{f}_0 \\
        \hat{f}_1 \\
        \hat{f}_2 \\
        \vdots    \\
        \hat{f}_{n-1}
    \end{bmatrix}
            & =
    \begin{bmatrix}
        1      & 1            & 1               & \cdots & 1                   \\
        1      & \omega^1     & \omega^2        & \cdots & \omega^{n-1}        \\
        1      & \omega^2     & \omega^4        & \cdots & \omega^{2(n-1)}     \\
        \vdots & \vdots       & \vdots          & \ddots & \vdots              \\
        1      & \omega^{n-1} & \omega^{2(n-1)} & \cdots & \omega^{(n-1)(n-1)}
    \end{bmatrix}
    \begin{bmatrix}
        f_0    \\
        f_1    \\
        f_2    \\
        \vdots \\
        f_{n-1}
    \end{bmatrix},
\end{align}
where $\omega = \exp\left(-2\pi i / n\right)$ is the primitive $n$th root of unity. Note that this makes $F_n$ a complex matrix. This operation is $O(n^2)$ when implemented naively, but can be brought down to $O(n \log n)$ using a divide-and-conquer approach: the fast Fourier transform (FFT).

\begin{figure}
    \centering
    \includegraphics[width=0.40\linewidth]{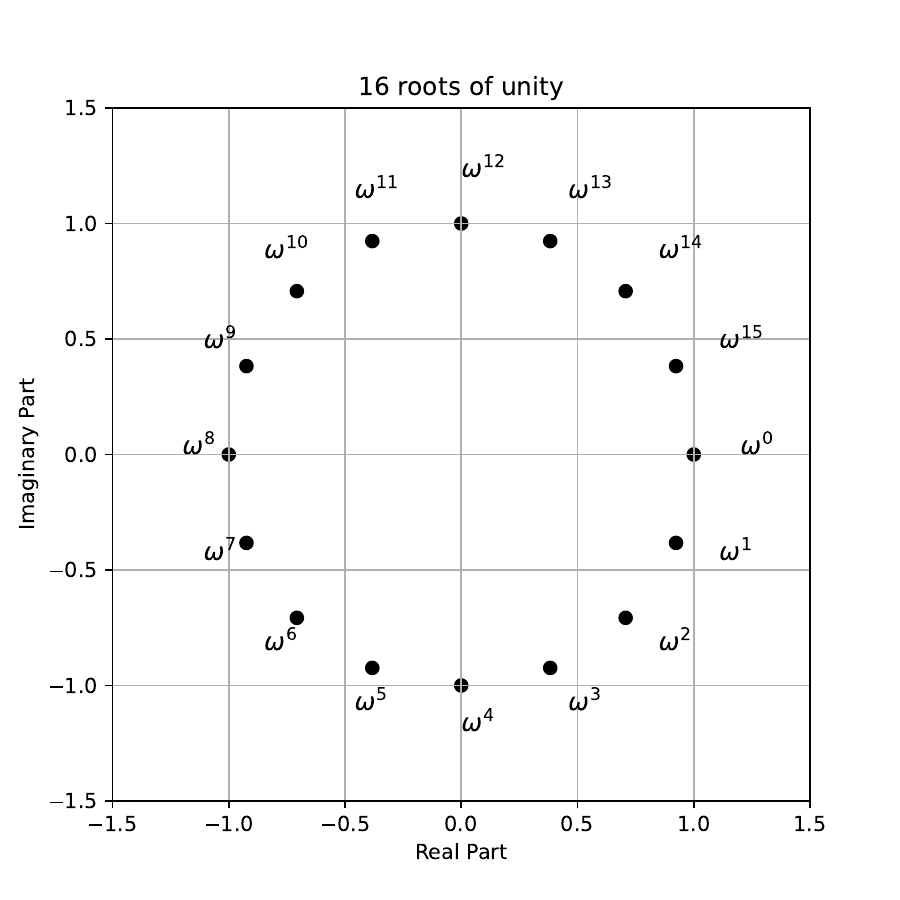}
    \includegraphics[width=0.55\linewidth]{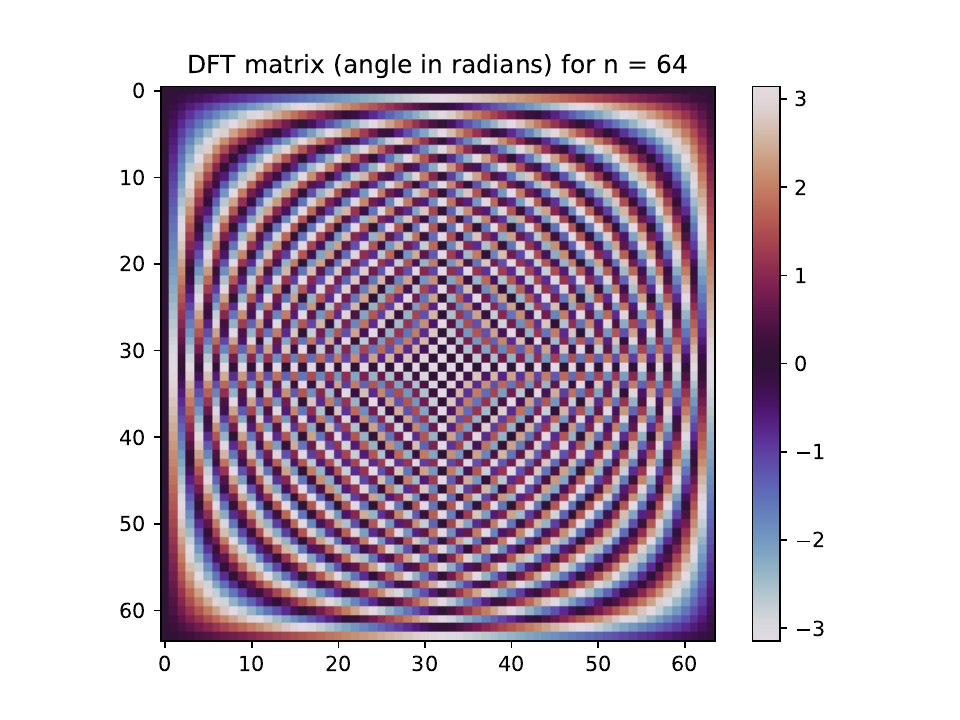}
    \caption{Symmetries of the discrete Fourier transform. (Left) Roots of unity in the complex plane. (Right) Structure of the DFT matrix illustrating periodicity and symmetry.}
    \label{fig:dft}
\end{figure}

To get an intuition, let us examine the case $n = 4$ first. Equation \eqref{eq:dft} expresses $\hat{f}$ in terms of $f$ and powers of $\omega = -i$. We want to ``divide'' this problem, and express it in terms of the DFT for $n = 2$, that uses powers of $\omega^2 = -1$. Thus,

\begin{align}
    \begin{bmatrix}
        \hat{f}_0 \\
        \hat{f}_1 \\
        \hat{f}_2 \\
        \hat{f}_{3}
    \end{bmatrix}
     & =
    \begin{bmatrix}
        1 & 1          & 1          & 1          \\
        1 & \omega^1   & \omega^2   & \omega^{3} \\
        1 & \omega^2   & \omega^4   & \omega^{6} \\
        1 & \omega^{3} & \omega^{6} & \omega^{9}
    \end{bmatrix}
    \begin{bmatrix}
        f_0 \\
        f_1 \\
        f_2 \\
        f_{3}
    \end{bmatrix}              \\
     & =
    \begin{bmatrix}
        \omega^0 & \omega^0 & \omega^0   & \omega^0   \\
        \omega^0 & \omega^1 & \omega^2   & \omega^{3} \\
        \omega^0 & \omega^2 & \omega^0   & \omega^{2} \\
        \omega^0 & \omega^3 & \omega^{2} & \omega^{1}
    \end{bmatrix}
    \begin{bmatrix}
        f_0 \\
        f_1 \\
        f_2 \\
        f_3
    \end{bmatrix}
     & \text{($\omega^4 = 1$)}.
\end{align}

Observe that $F_4$ exhibits a very clear structure. Row $i$ starts with $\omega^0 = 1$, which gets multiplied by $\omega^i$ repeatedly. The key step is realizing that if we take every other element from each row, we will get $F_2$. This makes a lot of sense, because we are just walking the unit circle clockwise twice as fast. We can express this as follows:

\begin{align}
    \begin{bmatrix}
        \hat{f}_0 \\
        \hat{f}_1 \\
        \hat{f}_2 \\
        \hat{f}_3
    \end{bmatrix}
     & =
    \begin{bmatrix}
        \omega^0 & \omega^0 & \omega^0 & \omega^0 \\
        \omega^0 & \omega^2 & \omega^1 & \omega^3 \\
        \omega^0 & \omega^0 & \omega^2 & \omega^2 \\
        \omega^0 & \omega^2 & \omega^3 & \omega^1 \\
    \end{bmatrix}
    \begin{bmatrix}
        f_0 \\
        f_2 \\
        f_1 \\
        f_3 \\
    \end{bmatrix} \\
     & =
    \begin{bmatrix}
        F_2 & \begin{array}{cc} \omega^0 & \omega^0 \\ \omega^1 & \omega^3 \end{array} \\
        F_2 & \begin{array}{cc} \omega^2 & \omega^2 \\ \omega^3 & \omega^1 \end{array}
    \end{bmatrix}.
\end{align}

What about the remaining columns? They will be like $F_2$, but each row will be offset by increasing powers of $\omega$: $\omega^0$, $\omega^1$, $\omega^2$, and $\omega^3$. Because $\omega^2 = -1$, we can re-write this sequence as $\omega^0$, $\omega^1$, $-\omega^0$, and $-\omega^1$. Finally,  we obtain that

\begin{align}
    \begin{bmatrix}
        \hat{f}_0 \\
        \hat{f}_1 \\
        \hat{f}_2 \\
        \hat{f}_{3}
    \end{bmatrix}
    =
    \begin{bmatrix}
        F_2 & D_2 F_2  \\
        F_2 & -D_2 F_2
    \end{bmatrix}
    \begin{bmatrix}
        f_0 \\
        f_2 \\
        f_1 \\
        f_3 \\
    \end{bmatrix},
\end{align}
where $D_2$ just multiplies each row by increasing factors of $\omega$:
\begin{equation}
    D_2 =
    \begin{bmatrix}
        1 & 0      \\
        0 & \omega
    \end{bmatrix}.
\end{equation}
Notice that there was nothing special about $n = 4$, and we can apply the same reasoning to conclude that if $n$ is a power of two,
\begin{align}
    \hat{f}
     & =\begin{bmatrix}
            F_{n/2} & D_{n/2} F_{n/2}  \\
            F_{n/2} & -D_{n/2} F_{n/2}
        \end{bmatrix}
    \begin{bmatrix}
        f_\text{even} \\
        f_\text{odd}  \\
    \end{bmatrix}                 \\
     & =
    \begin{bmatrix}
        F_{n/2} f_\text{even} + D_{n/2} F_{n/2}  f_\text{odd} \\
        F_{n/2} f_\text{even} - D_{n/2} F_{n/2}  f_\text{odd} \\
    \end{bmatrix},
\end{align}
where
\begin{equation}
    D_{n/2} =
    \begin{bmatrix}
        1      & 0      & 0        & \cdots & 0              \\
        0      & \omega & 0        & \cdots & 0              \\
        0      & 0      & \omega^2 & \cdots & 0              \\
        \vdots & \vdots & \vdots   & \ddots & \vdots         \\
        0      & 0      & 0        & \cdots & \omega^{n/2-1}
    \end{bmatrix}.
\end{equation}

This involves two operations: evaluating two DFTs on inputs that are half the size, and combining those results (which is linear because $D_{n/2}$ is diagonal). Using the master theorem \cite{cormen2022introduction}, we easily obtain that the runtime of the fast Fourier transform is $O(n \log n)$. Note that because of Theorem \ref{thm:inverse_fourier}, we can also compute the inverse Fourier transform in $O (n\log n)$ time.

\subsection{The Convolution Theorem}

\begin{definition}[Circular Convolution]
    \label{def:convolution}
    Let $f, g \in \mathbb{C}^n$. Their circular convolution is given by
    \begin{equation}
        (f * g)_k = \sum_{j = 0}^{n - 1} f_j g_{k - j},
    \end{equation}
    where the indexes are taken modulo $n$.
\end{definition}

Performing a convolution naively requires $O(n^2)$ operations. The following theorem will allow us to compute the convolution in the frequency domain in linear time. If we also consider the run-time of the FFT (and the inverse FFT), we conclude that convolutions can be computed in $O(n \log n)$ time.

\begin{theorem}[Convolution Theorem]
    \label{thm:convolution}
    Convolution in the time domain corresponds to point-wise multiplication in the frequency domain:
    \begin{equation}
        \widehat{f * g} = \hat{f} \cdot \hat{g}.
    \end{equation}
\end{theorem}
\begin{proof}
    Note: $g_i$ with $i < 0$ ``wraps-around'', so $-1$ would be $n - 1$ for example.
    \begin{align}
        (\widehat{f * g})_k & = \sum_{j = 0}^{n - 1} (f * g)_j \exp\left(-ikj \frac{2 \pi}{n}\right)                                                            \\
                            & = \sum_{j = 0}^{n - 1} \sum_{l = 0}^{n - 1} f_l g_{j - l} \exp\left(-ikj \frac{2 \pi}{n}\right)                                   \\
                            & = \sum_{l = 0}^{n - 1} f_l \sum_{j = 0}^{n - 1} g_{j - l} \exp\left(-ikj \frac{2 \pi}{n}\right)                                   \\
                            & = \sum_{l = 0}^{n - 1} f_l \sum_{m = 0}^{n - 1} g_{m} \exp\left(-ik(l + m) \frac{2 \pi}{n}\right)                                 \\
                            & =\sum_{l = 0}^{n - 1} f_l  \exp\left(-ikl \frac{2 \pi}{n}\right) \sum_{m = 0}^{n - 1} g_{m} \exp\left(-ikm \frac{2 \pi}{n}\right) \\
                            & = \hat{f}_k \cdot \hat{g}_k.
    \end{align}
\end{proof}

\subsection{Short-Time Fourier Transform}
The Fourier transform, and by extension the discrete Fourier transform, provides a frequency representation of an entire signal. However, many signals, particularly audio, are non-stationary, meaning their frequency content changes over time. For instance, in speech, different phonemes have different frequency characteristics, and in music, notes and chords change over time. Applying a single Fourier transform over the entire duration of such a signal would average out these temporal variations, losing critical information. The core idea of the short-time Fourier transform is to break the signal into short segments, and compute the FFT for each segment, thus capturing how the frequency content changes over time. A similar technique, the modified discrete cosine transform (MDCT), is widely used for audio compression.

\subsubsection{The Need for Windowing}

When we extract a segment (or ``frame'') from a longer signal to analyze its local frequency content using the DFT, we are implicitly assuming that this finite segment represents one period of an infinitely periodic signal. If the extracted segment is simply truncated (i.e., cut abruptly at its ends, which is equivalent to applying a rectangular window), sharp discontinuities can be created at the boundaries of this implied periodic repetition.

These discontinuities lead to a phenomenon called spectral leakage. In the frequency domain, this results in energy from the true frequency component ``leaking'' into adjacent frequency bins. To mitigate this, we apply a window function to each segment before computing its FFT. The window function is often zero or near-zero at the ends of the segment, which reduces discontinuities. See Figure \ref{fig:windowing} for a visual explanation.

\begin{figure}[h]
    \centering
    \begin{subfigure}{0.32\textwidth}
        \centering
        \includegraphics[width=\linewidth]{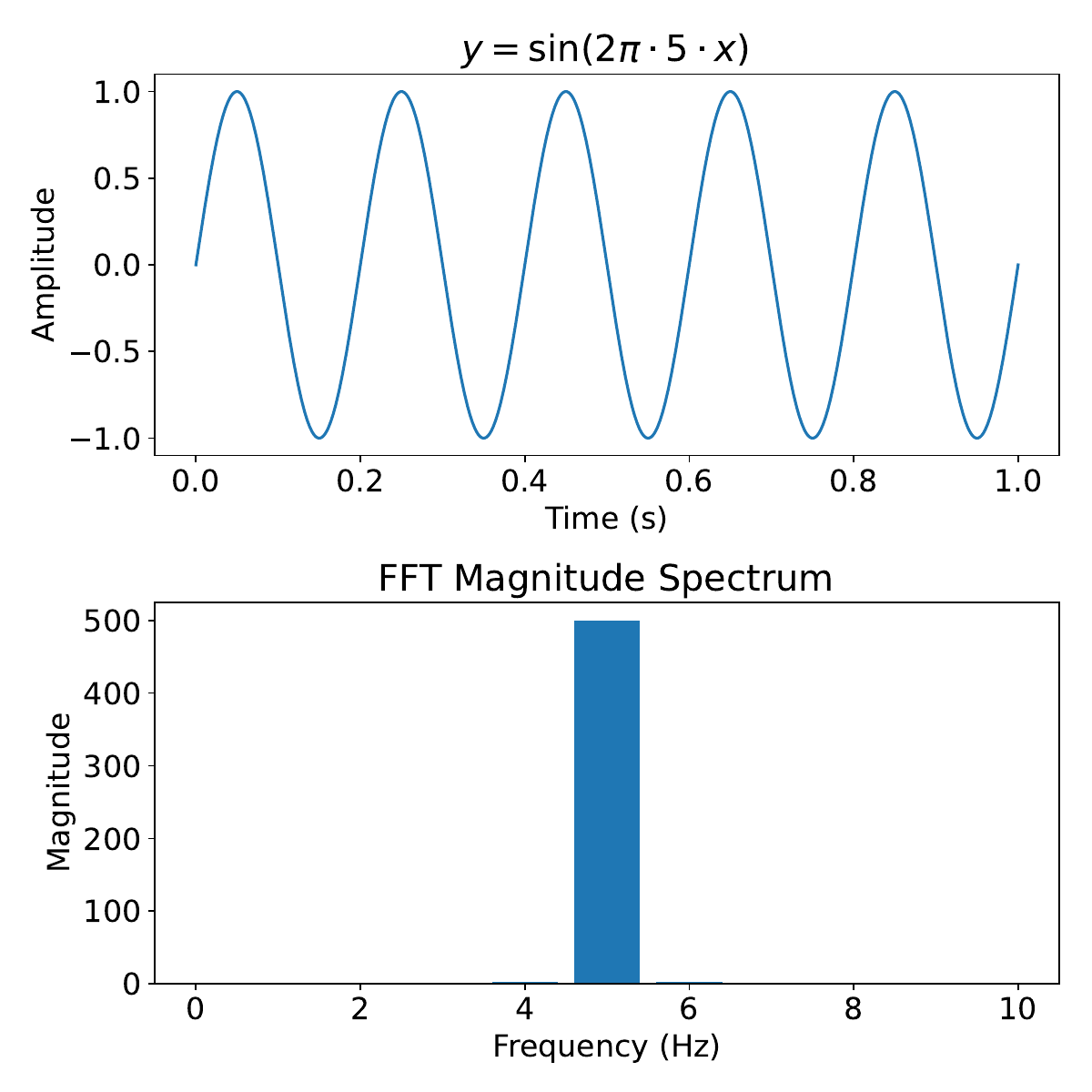}
        \caption{Rectangular window (whole number of periods)}
    \end{subfigure}
    \hfill
    \begin{subfigure}{0.32\textwidth}
        \centering
        \includegraphics[width=\linewidth]{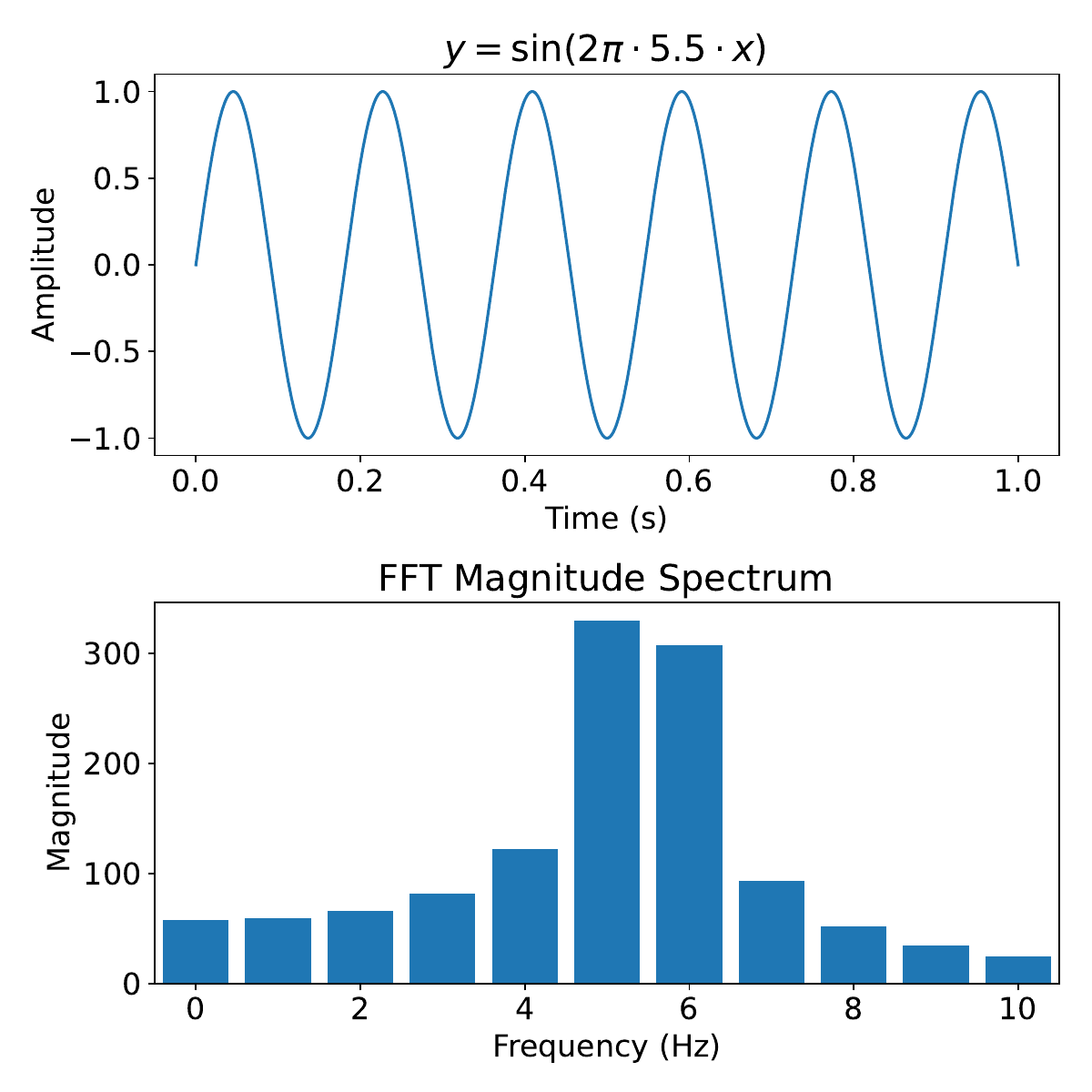}
        \caption{Rectangular window (fractional number of periods)}
    \end{subfigure}
    \hfill
    \begin{subfigure}{0.32\textwidth}
        \centering
        \includegraphics[width=\linewidth]{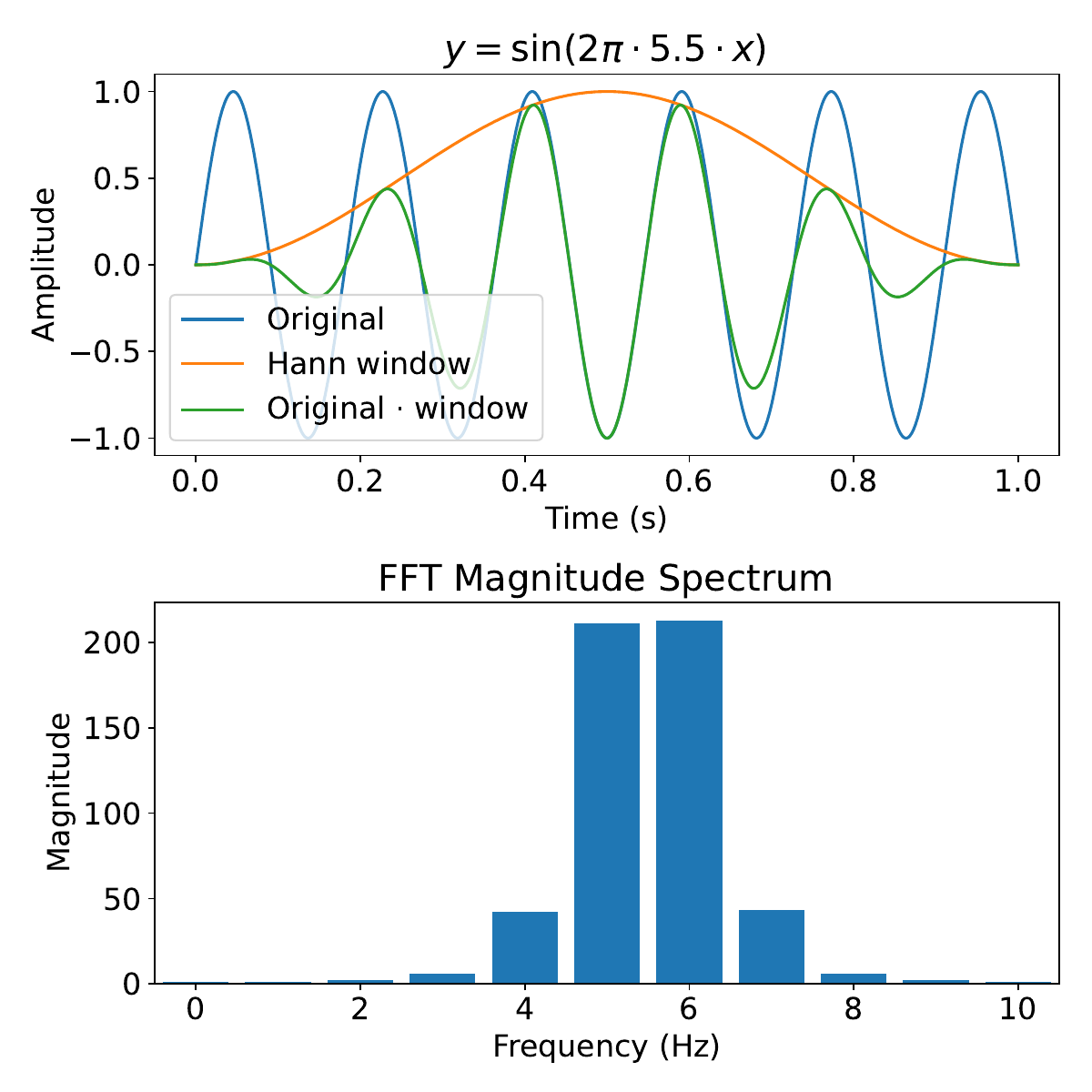}
        \caption{Hahn window (fractional number of periods)}
    \end{subfigure}
    \caption{\textbf{Effect of windowing on the STFT.} Figure (a) shows a signal with exactly five periods of a sine wave, where the start and end points align smoothly. As a result, the DFT clearly identifies the 5 Hz frequency. In Figure (b), the signal contains five and a half periods, and the endpoints do not match. This causes spectral leakage, spreading power into neighboring frequency bins beyond just 5 and 6 Hz. In Figure (c), the signal is multiplied by a Hann window, which reduces the discontinuities at the boundaries. Consequently, the magnitude spectrum becomes more concentrated around the 5 and 6 Hz components.}
    \label{fig:windowing}
\end{figure}

\begin{definition}[Discrete short-time Fourier transform] Let
    \begin{itemize}
        \item $f \in \mathbb{C}^S$ be a signal
        \item $n$ be the window size (or frame length)
        \item $w \in \mathbb{R}^n$ be a window function
        \item $h$ be the ``hop-size'' (often $\frac{n}{2}$ or $\frac{n}{4}$)
    \end{itemize}
    Then we can define
    \begin{equation}
        \hat{f}_{m, k} = \sum_{j = 0}^{n - 1} w_j \cdot f_{j + m h} \cdot \exp\left(-ikj \frac{2\pi}{n}\right),
    \end{equation}
    where
    \begin{itemize}
        \item $m$ ranges from $0$ to $\left\lfloor \frac{S - n}{h} \right\rfloor$
        \item $k$ ranges from $0$ to $n - 1$
    \end{itemize}
\end{definition}

There are two important hyperparameters here, the frame length $n$ and the hop-size $h$. The frame length introduces a time-frequency trade-off: larger values will result in good frequency resolution, while small values will provide good time localization.

\subsection{Mel-spectrograms}
The short-time Fourier transform of an audio signal yields a complex-valued matrix, with time along the x-axis and frequency along the y-axis. For visualization purposes, we typically extract the magnitude (or power) of each complex number and discard the phase. This is not because phase is unimportant (it is needed to reconstruct the signal) but because phase spectrograms often appear noisy and unstructured (see Figure \ref{fig:mel}).

Humans do not perceive frequency on a linear scale. In particular, we are more sensitive to changes at lower frequencies than at higher ones. To model this, the mel scale was introduced by \cite{mel}. The mel scale maps frequency (in Hz) to a perceptual scale (mels) where equal intervals correspond to equal perceived pitch differences. Although the mel scale was empirically derived from perceptual experiments, some researchers have proposed formulas to convert $f$ hertz into $m$ mels, such as \cite{o1987speech}:
\begin{equation}
    m = 2595 \log_{10}\left(1 +\frac{f}{700}\right).
\end{equation}

\begin{figure*}
    \centering
    \begin{subfigure}[b]{0.475\textwidth}
        \centering
        \includegraphics[width=\textwidth]{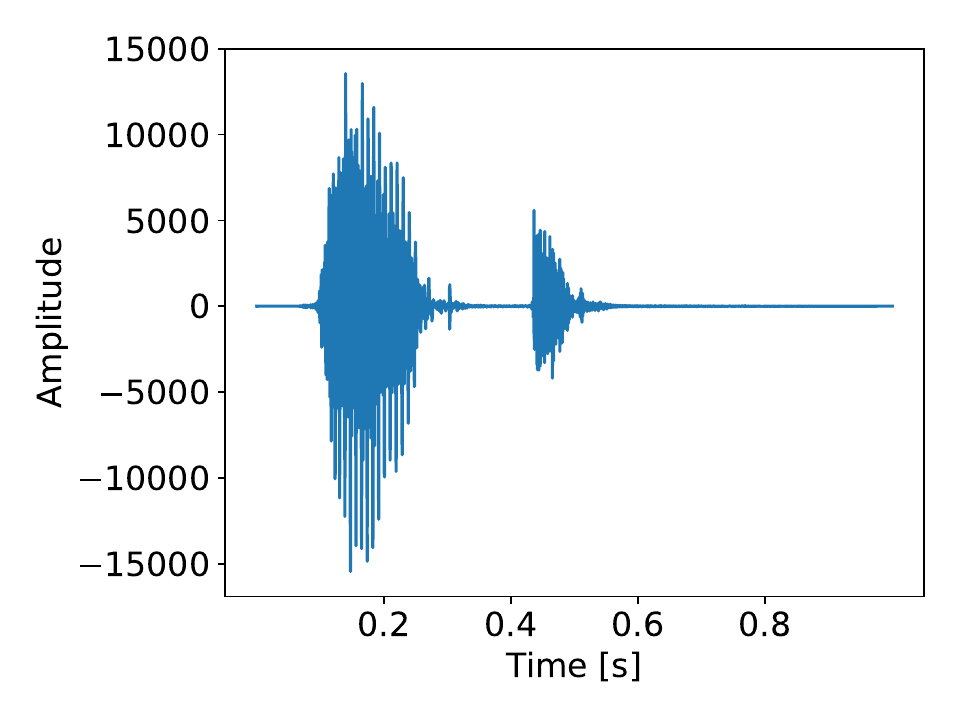}
        \caption{Original waveform}
    \end{subfigure}
    \hfill
    \begin{subfigure}[b]{0.475\textwidth}
        \centering
        \includegraphics[width=\textwidth]{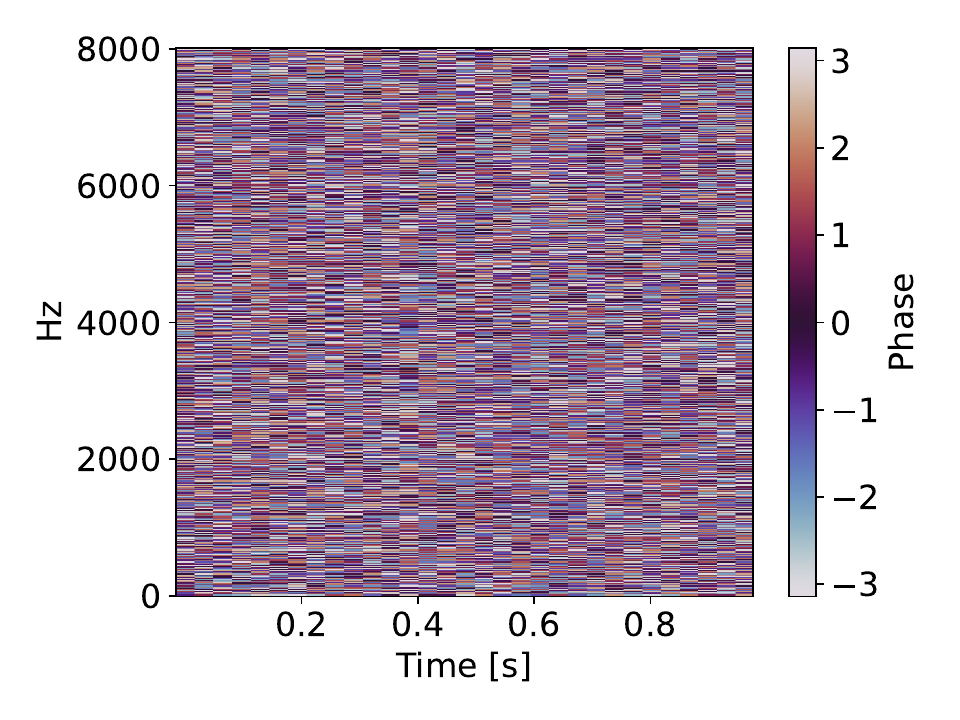}
        \caption{Phase of the STFT}
    \end{subfigure}
    \vskip\baselineskip
    \begin{subfigure}[b]{0.475\textwidth}
        \centering
        \includegraphics[width=\textwidth]{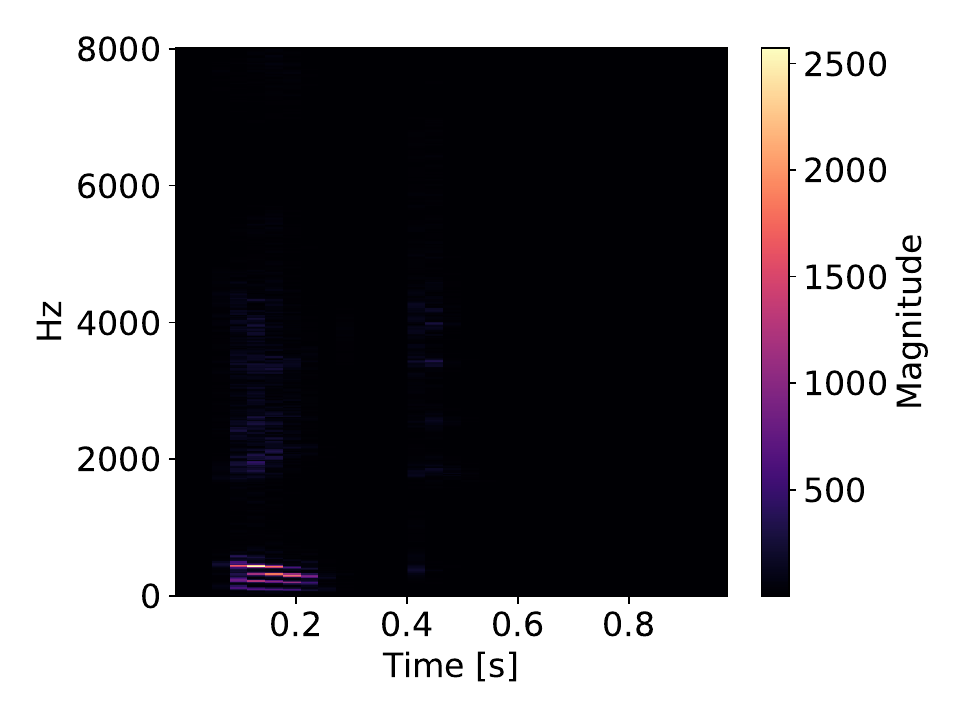}
        \caption{Magnitude of the STFT}
    \end{subfigure}
    \hfill
    \begin{subfigure}[b]{0.475\textwidth}
        \centering
        \includegraphics[width=\textwidth]{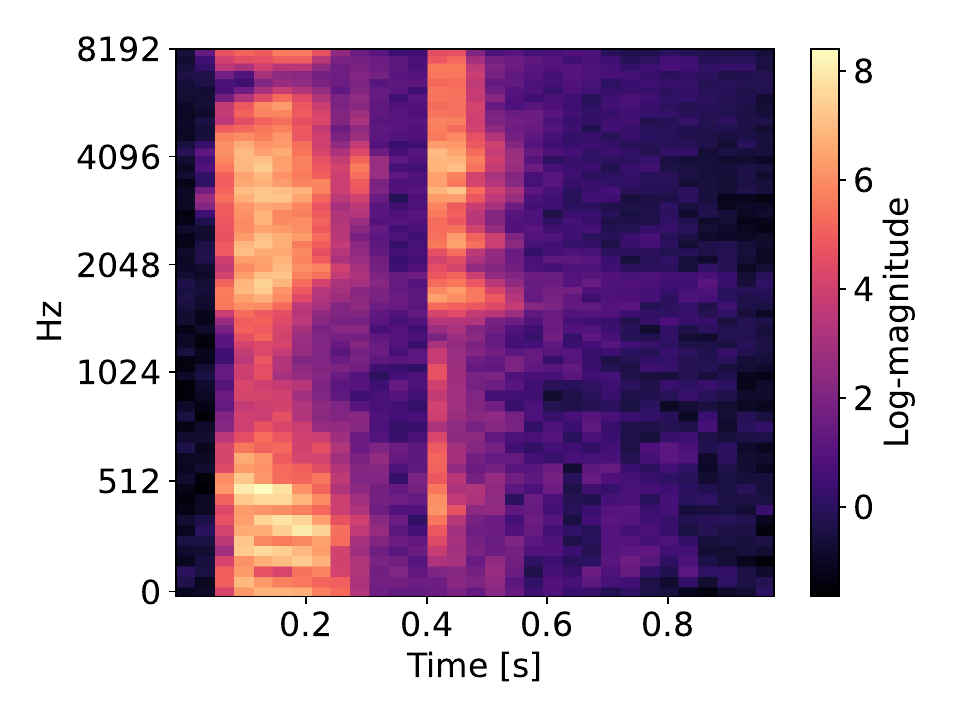}
        \caption{Mel-spectrogram}
    \end{subfigure}
    \caption{In (a) we see a waveform that corresponds to a one-second utterance from the SC09  dataset (see Section \ref{sec:datasets}). After applying the short-term Fourier transform (STFT), we get a complex-valued matrix. The phase (b) seems quite random, but the magnitude (c) is very structured. Finally, (d) shows the mel-spectrogram.}
    \label{fig:mel}
\end{figure*}

\section{Metrics}
For training, we maximize the log-likelihood of the ground-truth sequence \( x \) under the model's predicted distribution.

\begin{definition}[Log-likelihood]
    Given a ground-truth sequence \( x = (x_1, \dots, x_S) \), and a model \(P_\theta(\cdot \mid x_{<k}) \), the log-likelihood is defined as:
    \begin{align}
        \label{eq:ll}
        \LL(x; \theta)
         & = \log_2\left(\prod_{k = 1}^S P_{\theta}\left(x_k \mid x_{<k}\right)\right) \\
         & = \sum_{k = 1}^S \log_2 P_\theta(x_k \mid x_{<k}).
    \end{align}
\end{definition}

In order to evaluate our model, we use the log-likelihood on a held-out test dataset. In addition, and following prior literature, we also compute the FID and IS scores on samples generated by the model. In the generative case, we could compute $\LL(x; \theta)$ with samples instead of with the ground truth. However, this approach fails to capture sample diversity, as the model could repeatedly generate the same output. To better evaluate both the quality and diversity of the generated samples, we use the following procedure:
\begin{enumerate}
    \item \textbf{Train a classifier on the raw data.} For example, in the SC09 dataset (see Section \ref{sec:datasets}), we can classify each utterance into ten different digits using mel-spectrograms. In ImageNet, we can do 1000-class classification.
    \item \textbf{Use the activations of the classifier}. Two audio recordings that are semantically similar can be quite different in wave space. Because of this, we use some activations of the classifier. For example, the IS score uses the model's output probability distribution, and the FID score uses the last layer activations to compare the training samples with the generated ones.
\end{enumerate}
Before we dive into the specific metrics, we should mention some of their limitations. First, they heavily rely on the accuracy of the pre-trained classifier. Second, they assess the generated samples through a proxy, which might not fully align with human perception. Finally, these metrics might ignore low-level details, potentially allowing artifacts to go unnoticed.

\subsection{IS score}
The Inception Score \cite{is} was introduced in 2016 to evaluate the image quality of Generative Adversarial Networks (GANs). The classifier used was Inception v3 \cite {inceptionv3} (thus the name), and the activations used are the outputs of the model after the softmax layer.
\begin{definition}
    The IS score is defined as
    \begin{align}
        \ln \IS & = \EX [ \KL ( p(\cdot \mid X) \mid\mid \EX[p(\cdot \mid X)] )]                                                                \\
                & = \EX_{x \sim \text{samples}}  \left[ \KL ( p(\cdot \mid x) \mid\mid \EX_{x' \sim \text{samples}}[p(\cdot \mid x')] ) \right]
        ,
    \end{align}
    where $p(\cdot \mid x)$ is the probability distribution outputted by the classifier given $x$.
\end{definition}
Let us enumerate first some useful facts, that will be then proven.
\begin{itemize}
    \item We want to \textbf{maximize} the IS score.
    \item The IS score ranges between 1 (worst) and $K$ (best), the number of classes of the classifier.
    \item A score of one corresponds to the classifier always making the same prediction. This is bad because samples are not diverse.
    \item A score of $K$ means that the samples are evenly spread across all classes (good diversity) and that the classifier is always certain about its predictions (generated samples are sharp and distinct).
\end{itemize}

\begin{proposition}
    The IS score is always greater or equal to one
\end{proposition}
\begin{proof}
    Since the KL divergence is always non-negative, $\ln(\IS)$  will never be negative. By exponentiating, we get that $\IS \geq 1$.
\end{proof}

\begin{proposition}
    The IS score is bounded by $K$ and the bound is achieved when
    \begin{itemize}
        \item $\EX[p(\cdot \mid X)]_k = \frac{1}{K} \quad \forall k$ (samples spread across all classes)
        \item $p(\cdot \mid X)$ is one-hot (classifier is certain about its predictions)
    \end{itemize}
\end{proposition}
\begin{proof}
    Let $Y = p(\cdot \mid X)$ for clarity.
    \begin{align}
        \ln \IS & = \EX [ \KL ( p(\cdot \mid X) \mid\mid \EX[p(\cdot \mid X)] )]                                                                                           \\
                & = \EX [ \KL ( Y \mid\mid \EX[Y] )]                                                                                                                       \\
                & = \EX \left[ \sum_{k=1}^K  y_k  \log\left(\frac{y_k}{\EX[Y]_k}\right) \right]                                                                            \\
                & = \EX \left[ \sum_{k=1}^K  y_k  \log(y_k)\right] - \EX \left[ \sum_{k=1}^K  y_k  \log(\EX[Y]_k)\right]                                                   \\
                & = -\EX [H(Y)] -  \sum_{k=1}^K  \EX [Y]_k  \log(\EX[Y]_k)                                               & \text{(definition of Entropy)}                  \\
                & = -\EX [H(Y)] + H(\EX[Y])                                                                              & \text{(definition of Entropy)}                  \\
                & \leq \log(K).                                                                                          & \text{(Entropy is bounded from 0 to $\log(K)$)}
    \end{align}
    We can just take exponentials on both sides to get the required bound. Let us now show that the bound can be achieved. If $\EX[Y]$ is uniform (samples are uniformly spread across all classes), $H(\EX[Y]) = \log(K)$. If the classifier is completely certain about its predictions, $Y$ will be one-hot, so $\EX[H(Y)] = 0$.
\end{proof}

\subsection{FID score}
The Fréchet Inception Distance (FID) \cite{fid}, introduced in 2017, was designed to overcome a key limitation of the Inception Score (IS): IS does not take into account the statistics of the real data distribution. FID addresses this by modeling both the real and generated samples as multivariate Gaussian distributions and comparing them using the Fréchet distance. Unlike the IS, where higher values are better, lower FID scores indicate that the samples are similar to the training data.

To make this comparison meaningful, samples are not evaluated in their raw form (e.g., waveform for audio or pixel space for images), but rather through their representations in a semantically rich space (the activations from the final layer of a pre-trained classifier).

\begin{definition}
    The Fréchet distance between two probability distributions $P$ and $Q$ is defined as
    \begin{equation}
        d(P, Q)^2 = \min_{X \sim P, Y \sim Q} \EX[\norm{X - Y}^2].
    \end{equation}
\end{definition}

\begin{proposition}
    \label{prop:centering}
    We can define $d(P, Q)^2$ in terms of the centered variables $\hat{X} = X - \mu$ and $\hat{Y} = Y - \mu'$:
    \begin{equation}
        d(P, Q)^2 = \norm{\mu - \mu'}^2 + \min_{X \sim P, Y \sim Q} \EX\left[\norm{\hat{X} - \hat{Y}}^2\right].
    \end{equation}
\end{proposition}
\begin{proof}
    \begin{align}
        \norm{X - Y}^2 & = \langle X - Y, X - Y\rangle                                                                       \\
                       & = \langle \hat{X} - \hat{Y} + \mu - \mu',  \hat{X} - \hat{Y} + \mu - \mu'\rangle                    \\
                       & = \norm{\hat{X} - \hat{Y}}^2 + \norm{\mu - \mu'}^2 +2\langle \hat{X} - \hat{Y}, \mu - \mu' \rangle.
    \end{align}
    Notice that the third term will cancel after taking expectations, because the expected value of the centered variables is zero. Thus,
    \begin{align}
        d(P, Q)^2 & = \min_{X \sim P, Y \sim Q} \EX[\norm{X - Y}^2]                                                \\
                  & =  \norm{\mu - \mu'}^2 + \min_{X \sim P, Y \sim Q} \EX\left[\norm{\hat{X} - \hat{Y}}^2\right].
    \end{align}

\end{proof}

\begin{proposition}
    \label{prop:expectation_normal_norm_squared}
    If $Z \sim \mathcal{N}(\bm{0}, \Sigma)$, $\EX[\norm{Z}^2] = \Tr(\Sigma)$.
\end{proposition}
\begin{proof}
    \begin{align}
        \EX[\norm{Z}^2] = \EX\left[\sum_i Z_i^2\right] = \sum_i \EX[Z_i^2] = \sum_i \Var(Z_i) = \Tr(\Sigma).
    \end{align}

\end{proof}

\begin{theorem}
    Let $P = \mathcal{N}(\mu, \Sigma)$ and $Q = \mathcal{N}(\mu', \Sigma')$, with $\Sigma$ and $\Sigma'$ invertible. Then,
    \begin{equation}
        d(P, Q)^2 = \norm{\mu - \mu'}^2 + \Tr \left(\Sigma + \Sigma' - 2\sqrt{\sqrt{\Sigma} \Sigma' \sqrt{\Sigma}} \right).
    \end{equation}
\end{theorem}
\begin{proof}
    We will use \cref{prop:centering} to assume without loss of generality that $X$ and $Y$ are centered around zero. Thus, we have
    \begin{align}
        d(P, Q)^2 & = \min_{X \sim P, Y \sim Q} \EX[\norm{X - Y}^2] & \text{(\cref{prop:centering})}                       \\
                  & =  \min_{X \sim P, Y \sim Q} \Tr(\Cov(X - Y)).  & \text{(\cref{prop:expectation_normal_norm_squared})}
    \end{align}
    Note that
    \begin{equation}
        X - Y =
        \begin{bmatrix}
            I & -I
        \end{bmatrix}
        \begin{bmatrix}
            X \\
            Y
        \end{bmatrix},
    \end{equation}
    so
    \begin{equation}
        \Cov(X - Y) =
        \begin{bmatrix}
            I & -I
        \end{bmatrix}
        \begin{bmatrix}
            \Sigma & C       \\
            C^T    & \Sigma'
        \end{bmatrix}
        \begin{bmatrix}
            I \\
            -I
        \end{bmatrix}
        =
        \Sigma + \Sigma' - C - C^T,
    \end{equation}
    where $C$ is the cross-covariance matrix. Thus, we have to
    \begin{equation}
        \min_{X \sim P, Y \sim Q} \Tr\left(\Sigma + \Sigma' - C - C^T\right).
    \end{equation}
    Note that $\Sigma$ and $\Sigma'$ are fixed, and $\Tr C = \Tr C^T$, so we just have to maximize $\Tr C$ under the constraint that $\Cov(X - Y)$ should be positive semi-definite:
    \begin{equation}
        \max_{\Cov(X - Y) \succeq 0 } \Tr C.
    \end{equation}
    Because of Schur's condition for positive definiteness,
    \begin{equation}
        \Cov(X - Y)  =
        \begin{bmatrix}
            \Sigma & C       \\
            C^T    & \Sigma'
        \end{bmatrix}
        \succeq
        0
        \quad \iff \quad
        \Sigma' - C^T \Sigma^{-1} C \succeq 0.
    \end{equation}
    Now we will substitute $C = \sqrt{\Sigma} K \sqrt{\Sigma'}$, which we can do because both  $\sqrt{\Sigma}$ and $\sqrt{\Sigma'}$ are invertible. While this might seem a bit obscure at first, it will greatly simplify things:
    \begin{align}
        \Sigma' - C^T \Sigma^{-1} C \succeq 0  \quad & \iff \quad \Sigma' - \sqrt{\Sigma'} K^T K \sqrt{\Sigma'} \succeq 0                                     \\
        \quad                                        & \iff \quad \sqrt{\Sigma'} (I - K^T K) \sqrt{\Sigma'} \succeq 0                                         \\
        \quad                                        & \iff \quad (I - K^T K) \succeq 0                                   & \text{(congruent transformation)} \\
        \quad                                        & \iff \quad \norm{K} \leq 1.                                        & \text{(spectral norm definition)}
    \end{align}
    Our problem is now
    \begin{equation}
        \max_{\norm{K} \leq 1 } \Tr \left( {\sqrt{\Sigma} K \sqrt{\Sigma'}} \right).
    \end{equation}
    So,
    \begin{align}
        \max_{\norm{K} \leq 1 } \Tr \left( {\sqrt{\Sigma} K \sqrt{\Sigma'}} \right) & = \max_{\norm{K} \leq 1 } \Tr \left(K  \underbrace{{\sqrt{\Sigma'} \sqrt{\Sigma}}}_A \right) & \text{(cyclic property of the trace)}                      \\
                                                                                    & = \max_{\norm{K} \leq 1 } \Tr \left(K  U D V^T \right)                                       & \text{(SVD decomposition of $A$)}                          \\
                                                                                    & = \max_{\norm{K} \leq 1 } \Tr \left(\underbrace{V^T K U}_M D \right)                         & \text{(cyclic property of the trace)}                      \\
                                                                                    & = \max_{\norm{K} \leq 1 } \sum_i M_{ii} D_{ii}                                               & \text{(definition of trace)}                               \\
                                                                                    & \leq \sum_i D_{ii}                                                                           & \text{(norm is unitarily invariant, so $\norm{M} \leq 1$)} \\
                                                                                    & = \Tr(D)                                                                                     & \text{(definition of trace)}                               \\
                                                                                    & = \Tr(\sqrt{A^T A})                                                                          & \text{($A = U D V^T$, so $\sqrt{A^T A} = V D V^T$)}        \\
                                                                                    & = \Tr\left(\sqrt{\sqrt{\Sigma} \Sigma' \sqrt{\Sigma}}\right).
    \end{align}
    It is not immediately clear why $M_{ii} \leq 1$. However, note that $\norm{M} = \norm{V^T K U} = \norm{K} \leq 1$. Thus,
    \begin{equation}
        \frac{x^T M x}{x^T x} \leq 1 \quad \forall x.
    \end{equation}
    In particular, if we choose $x = e_i$, we get $M_{ii} \leq 1$. We established an upper bound, that can be easily achieved setting $M = I$, which is the same as choosing $K = V U^{T}$.

\end{proof}

\section{Sequence-to-sequence models}
Our goal is to maximize the log-likelihood, so we need a model that receives a variable number of tokens $x_1, ..., x_{k - 1}$, and outputs $P_\theta(x_k \mid x_{<k})$. In particular, we will focus on sequence-to-sequence models, whose inputs and outputs have shape $(B, S, D)$, where $B$ represents the batch size, $S$ the sequence length and $D$ the feature dimension. The tensor \texttt{output[:, k, :]} is used to compute the conditional probability $P_\theta(x_k \mid x_{<k})$, as required by Equation~\eqref{eq:ll}.

These models are composed of several blocks, each of which receives and outputs tensors of shape $(B, S, D)$. There are two kinds of blocks that alternate:
\begin{enumerate}
    \item \textbf{Channel mixing.} This is typically done with multilayer perceptron (MLPs), that require
          \begin{itemize}
              \item $O(D^2)$ parameters
              \item $O(BSD^2)$ matmul flops
          \end{itemize}

    \item \textbf{Temporal mixing,} which mixes the activations across the time dimension. The complexity of this block depends on the layer used.
\end{enumerate}

As we will see in the next section, temporal mixing is $O(BSD^2)$ if we use recurrent neural networks, or $O(BS^2D + BSD^2)$ is we use self-attention. Thus, if $D \gg S$ (very often the case), it is perfectly reasonable to use self-attention. However, if we are dealing with very long sequences, recurrent networks are preferred, since they are linear with the sequence length.

\subsection{Recurrent neural networks}
Recurrent neural networks process data sequentially, keeping some sort of state. In each time-step, the output $y_k$ and the new state are computed based on the current input $x_k$, and state. We will now review several recurrent neural networks.

\begin{table}[h]
    \centering
    \begin{tabular}{lccc}
        \hline
        \textbf{Model} & \textbf{Parameters} & \textbf{Matmul flops} & \textbf{Element-wise flops} \\
        \hline
        Elman RNN      & $3D^2 + 2D$         & $6BSD^2$              & $5 BSD$                     \\
        LSTM           & $8D^2 + 4D$         & $16BSD^2$             & $17 BSD$                    \\
        GRU            & $6D^2 + 3D$         & $12BSD^2$             & $14 BSD$                    \\
        RG-LRU         & $2 D^2 + 3D$        & $4 BSD^2$             & $13 B S D$                  \\
        \hline
    \end{tabular}
    \caption{Complexity of recurrent neural network layers in terms of the batch size (B), the sequence length (S), and the number of channels (D). Note: multiplying a matrix and a vector requires $D(2D - 1)$ operations ($D$ multiplications and $D - 1$ sums are required for every output).}
    \label{tab:rnn-complexity}
\end{table}

\subsubsection{Elman network}
An Elman network \cite{elman} uses a single vector to keep track of the state.
\begin{align}
    h_k & = \sigma\left(W_h x_k + U_h h_{k - 1} + b_h\right) \\
    y_k & = \sigma\left(W_y h_k + b_y\right),
\end{align}
where $\sigma$ denotes the logistic function $\sigma(x) \coloneq 1 / (1 + \exp(-x))$.
\subsubsection{LSTM}
The long short-term memory unit \cite{lstm} was designed to mitigate the vanishing gradient problem. The state is comprised of two different vectors, the cell state $c_k$ and the hidden state $h_k$.
\begin{align}
    f_k         & = \sigma(W_{f} x_k + U_{f} h_{k-1} + b_f)   \\
    i_k         & = \sigma(W_{i} x_k + U_{i} h_{k-1} + b_i)   \\
    o_k         & = \sigma(W_{o} x_k + U_{o} h_{k-1} + b_o)   \\
    \tilde{c}_k & = \tanh(W_{c} x_k + U_{c} h_{k-1} + b_c)    \\
    c_k         & = f_k \odot c_{k-1} + i_k \odot \tilde{c}_k \\
    h_k         & = o_k \odot \tanh(c_k)                      \\
    y_k         & = h_k                                       \\
\end{align}
\subsubsection{GRU}
The gated recurrent unit \cite{gru} is similar to an LSTM, but the state is comprised of just one vector. It also requires fewer parameters.
\begin{align}
    z_k       & = \sigma(W_{z} x_k + U_{z} h_{k-1} + b_z)            \\
    r_k       & = \sigma(W_{r} x_k + U_{r} h_{k-1} + b_r)            \\
    \hat{h}_k & = \tanh(W_{h} x_k + U_{h} (r_k \odot h_{k-1}) + b_h) \\
    h_k       & =   (1-z_k) \odot  h_{k-1} + z_k \odot  \hat{h}_k
\end{align}
\subsubsection{RG-LRU}
The real-gated linear recurrent unit \cite{rglru} can be expressed as $h_k = M(x) h_{k - 1} + c(x)$. As we will explain in the state-space models section, this is crucial to parallelize the computation. In addition, it incorporates a gating mechanism, inspired by LSTMs and GRUs.
\begin{align}
    r_k & = \sigma(W_a x_k + b_a)                                        \\
    i_k & = \sigma(W_x x_k + b_x)                                        \\
    a_k & = a^{c r_k}                                                    \\
    h_k & = a_k \odot h_{k - 1} + \sqrt{1 - a_k^2} \odot (i_k \odot x_k)
\end{align}
$a$ can be parametrized as $\sigma(\Lambda)$, where $\Lambda \in \mathbb{R}^D$, or $\sigma(\Lambda) \exp(i \theta)$, where $\Lambda, \theta \in \mathbb{R}^\frac{D}{2}$. This guarantees that $0 \leq \norm{a_k} \leq 1$. $c$ is usually set to a scalar-valued constant.

\subsection{Self-attention}
Transformers \cite{attention} revolutionized the field of natural language processing by introducing the attention mechanism, that allows each token to attend to other tokens in the sequence. This contrasts with recurrent neural networks, that only have access to the current token and a compressed representation of the past (hidden state).

We will now give a brief explanation of the multi-head self-attention mechanism. Let $x \in \mathbb{R}^{S \times D}$ be a sample within a batch. A self-attention layer will take $x$ and output $y \in \mathbb{R}^{S \times D}$:
\begin{enumerate}
    \item First, we compute $q, k, v \in \mathbb{R}^{S \times H \times \frac{D}{H}}$, where $q, k$ and $v$ stand for queries, keys, and values, respectively. $H$ is the number of ``heads'', a hyperparameter. This requires
          \begin{itemize}
              \item $3D^2$ parameters
              \item $6 B S D^2$ matmul flops
          \end{itemize}

    \item We compute the attention matrix $a \in \mathbb{R}^{S \times S \times H}$, where
          \begin{equation}
              a_{ijh} = \frac{\exp\left( \tilde{a}_{ijh}\right)}{\sum\limits_{j'=1}^{S} \exp\left( \tilde{a}_{ij'h} \right)},
          \end{equation}
          and
          \begin{equation}
              \tilde{a}_{ijh} = \frac{\langle q_{i h :}, k_{j h :} \rangle}{\sqrt{D/H}}.
          \end{equation}
          In the autoregressive setting, the unnormalized tensor $\tilde{a}$ is modified by setting $\tilde{a}_{ijh} = -\infty$ whenever $j > i$ to prevent tokens from attending to the future. This requires
          \begin{itemize}
              \item No parameters
              \item $2 B S^2 D$ matmul flops
          \end{itemize}
    \item We compute the output of each attention head by applying the attention weights to the values. This yields a tensor $z \in \mathbb{R}^{S \times H \times \frac{D}{H}}$, where:
          \begin{equation}
              z_{ihd} = \sum_{j=1}^{S} a_{ijh} v_{jhd}.
          \end{equation}
          This step requires:
          \begin{itemize}
              \item No parameters
              \item $2 B S^2 D$ matmul flops
          \end{itemize}
    \item Finally, we concatenate the outputs of all $H$ heads and apply a linear projection to produce the output $y \in \mathbb{R}^{S \times D}$. This step requires:
          \begin{itemize}
              \item $D^2$ parameters
              \item $2 B S D^2$ matmul flops
          \end{itemize}
\end{enumerate}

Thus, the number of matmul flops is $O(BSD^2 + BS^2D)$, so it is quadratic both in the model's dimension, and in the sequence length. In many real world cases, $S \ll D$. For example, GPT-3 was trained with $S = 2048$ and $D = 12288$ \cite{gpt3}, so $S \ll D$. In this thesis, we focus on the opposite case.

Some modern transformer architectures make use of multi-query attention \cite{shazeer2019fast} instead of multi-head attention, in which the same keys and values are shared across all heads, reducing the size of these tensors. While this results in a significant speedup because of the lower memory requirements, it does not change the quadratic complexity with respect to sequence length.

\subsection{State Space Models and Linear RNNs}
State space models (SSMs) are a class of models originally developed in control theory to describe the evolution of dynamical systems in continuous time. In this framework, the system's state evolves according to a set of differential equations, and outputs are generated as functions of the state and input. In practice, we discretize these equations to obtain models suitable for digital computation. The resulting discrete-time SSMs are equivalent to recurrent neural networks (RNNs) with linear state transitions. We include the continuous case here for historical context and completeness, but our focus is on the discretized version used in modern sequence modeling.
\begin{definition}[Continuous-time SSM]
    A state space model is defined by the following equation, that maps an input signal $x(t)$ to a latent state $h(t)$ and an output $y(t)$.
    \begin{align}
        h'(t) & = A h(t) + B x(t) \\
        y(t)  & = C h(t) + D x(t) \\
    \end{align}
\end{definition}
\begin{definition}[Discrete-time SSM]
    The discrete SSM is given by
    \begin{align}
        h_k & = \overline{A} h_{k - 1} + \overline{B} x_k \\
        y_k & = \overline{C} h_{k} + \overline{D} x_k,
    \end{align}
    where $\overline{A}, \overline{B}, \overline{C}$, and $\overline{D}$ depend on both $A, B, C, D$, and $\Delta$ (the time-step).
\end{definition}

\subsubsection{Discretization}

How the continuous parameters are transformed into the discrete parameters depends on the discretization rule. Let us briefly go through the zero-order hold method. We will assume that $x(t)$ is piecewise constant, so $x(t) = x_k$ for all $t \in \left[k \Delta, (k + 1)\Delta\right)$. Note that $h(t)$ is not piecewise constant, and that to find it we need to solve the following ODE, where $t$ runs from 0 to $\Delta$ (for convenience), and $h(0) = h_k$.
\begin{equation}
    h'(t) = Ah(t) + B x_k.
\end{equation}
As usual with inhomogeneous linear ODEs, we will first find a homogeneous solution (setting $B x_k = 0$):
\begin{equation}
    h'(t) = A h(t) \implies h(t) = \exp(At) c,
\end{equation}
where $\exp(M) = I + M + \frac{1}{2!} M^2 + \frac{1}{3!} M^3 + \cdots$ is the matrix exponential. Then, we find a particular solution (setting $h'(t) = 0$):
\begin{equation}
    A h(t) + B x_k = 0 \implies h(t) = -A^{-1} B x_k.
\end{equation}
Thus, a general solution will be given by
\begin{equation}
    h(t) = \exp(At) c -A^{-1} B x_k,
\end{equation}
where $h(0) = h_k$. Imposing this yields $c = h_k + A^{-1} B x_k$. Thus,
\begin{equation}
    h(t) = \exp(At) \left(h_k + A^{-1} B x_k\right) -A^{-1} B x_k.
\end{equation}
Substituting $t = \Delta$ and rearranging yields
\begin{equation}
    h_{k + 1} = h(\Delta) = \exp(A \Delta) h_k + \left(\exp(A \Delta) - I\right) A^{-1} B x_k.
\end{equation}
Thus,
\begin{align}
    \overline{A} & = \exp(A \Delta)                           \\
    \overline{B} & = \left(\exp(A \Delta) - I\right) A^{-1} B
\end{align}
Note: in some papers in the literature \cite{mamba}, $A^{-1}$ is placed up front. Naturally, the expressions are identical, because the solution is unique (Picard–Lindelöf theorem). The reason is that $A^{-1}$ commutes with $\exp(A \Delta)$, as the exponential is polynomial in $A$.

\subsubsection{The recurrent view}
The discretized SSM is equivalent to a recurrent neural network, allowing us to easily compute $y$ given $x$ in $O(S D^2)$ time. We always use the recurrent view for inference. For training, we can use the recurrent view, the convolutional view, or an associative scan.
\subsubsection{The convolutional view}
The convolutional view can be used to parallelize training. Note that $y$ can easily be computed in parallel once we have $h$, so let's focus on $h$, which is defined sequentially. The sequence $h_k$ follows a simple recurrent neural network without gating, so we can unroll the recurrence as follows:
\begin{align}
    h_0 & \coloneq 0                                                                          \\
    h_1 & = \overline{B} x_1                                                                  \\
    h_2 & =\overline{B} x_2 + \overline{A} \overline{B} x_1                                   \\
    h_3 & =\overline{B} x_3 + \overline{A} \overline{B} x_2 + \overline{A}^2 \overline{B} x_1 \\
    \vdots                                                                                    \\
    h_k & = \sum_{i = 1}^k \overline{A}^{k - i} \overline{B} x_i                              \\
\end{align}
Note that $h = K * x$, where $*$ denotes convolution (Definition \ref{def:convolution}), and $K = (\overline{B}, \overline{A}\overline{B}, \overline{A}^2\overline{B}, ...)$. Thus, given $K$ we can compute $h$ in $O(S \log(S) D^2)$ using the FFT (Theorem \ref{thm:convolution}). This is slower than the recurrent implementation, but it is \emph{parallelizable}. In particular, with enough devices the parallel FFT can bring the time down to $O(\log(S) D^2)$.

The biggest problem is that in order to compute $K$, we have to multiply $\overline{A}$ with itself $S$ times, which results in a runtime of $O(S D^3)$. One way of speeding this up is by diagonalizing $\overline{A}$. Let us consider the change of basis $h = V h'$.

\begin{equation}
    \left\{
    \begin{aligned}
        h_k & = \overline{A} h_{k - 1} + \overline{B} x_k \\
        y_k & = \overline{C} h_{k} + \overline{D} x_k
    \end{aligned}
    \right.
    \, \Rightarrow \,
    \left\{
    \begin{aligned}
        V h_k' & = \overline{A} V h_{k - 1}' + \overline{B} x_k \\
        y_k    & = \overline{C} V h_{k}' + \overline{D} x_k
    \end{aligned}
    \right.
    \, \Rightarrow \,
    \left\{
    \begin{aligned}
        h_k' & = V^{-1} \overline{A} V h_{k - 1}' + V^{-1} \overline{B} x_k \\
        y_k  & = \overline{C} V h_{k}' + \overline{D} x_k
    \end{aligned}
    \right.
\end{equation}
The new recurrence matrix $V^{-1} \overline{A} V$ will be diagonal for some matrix $V$ if we require $\overline{A}$ to be diagonalizable. Fortunately, the set of diagonalizable matrices is dense, meaning that we can approximate any matrix with a diagonalizable one to arbitrary precision.

\subsubsection{The associative scan view}
So far, we have seen that the recurrent view is $O(S)$ but can't be parallelized, and that the convolutional view is $O(S \log S)$ but can be parallelized. In this section, we will see another way of achieving parallelization leveraging the associative scan. We would like to remark that this is only applicable to training, for inference, the recurrent view is always used.

As motivation, consider the problem of computing the cumulative sum of an array $a$. This task is $O(S)$, and it does not seem easily parallelizable, because in order to compute $\text{sum}(a[:n])$ one needs $\text{sum}(a[:n - 1])$. However, with enough processors one can bring down the runtime to $O(\log S)$ (see Figure \ref{fig:associative_scan} for an overview of the algorithm). Interestingly, this can be done with any associative binary operator. In the rest of this section we will ``lift'' an affine recurrence into a parallel scan.

\begin{definition}[Associative scan]
    Given an array of elements $[e_1, e_2, ..., e_n]$ with an associative binary operation $\bullet$ between them, an associative scan outputs the array $[e_1, e_1 \bullet e_2, ..., e_1 \bullet \dots \bullet e_n]$.
\end{definition}

\begin{figure}[h]
    \centering
    \textbf{Up Sweep} \\
    \begin{tikzpicture}[
            level distance=1.2cm,
            level 1/.style={sibling distance=6cm},
            level 2/.style={sibling distance=3cm},
            level 3/.style={sibling distance=1.5cm},
            every node/.style={draw, rectangle, minimum size=0.8cm, inner sep=2pt},
            edge from parent/.style={draw, thick, {Latex}-}
        ]

        \node (root) {25}
        child {node {11}
                child {node {4}
                        child {node {3}}
                        child {node {1}}
                    }
                child {node {7}
                        child {node {7}}
                        child {node {0}}
                    }
            }
        child {node {14}
                child {node {5}
                        child {node {4}}
                        child {node {1}}
                    }
                child {node {9}
                        child {node {6}}
                        child {node {3}}
                    }
            };
    \end{tikzpicture}

    \vspace{0.5em}
    \[
        \text{sum}[v] = \text{sum}[L[v]] + \text{sum}[R[v]]
    \]

    \vspace{1.5cm}
    \textbf{Down Sweep}

    \begin{tikzpicture}[
        level distance=1.2cm,
        level 1/.style={sibling distance=6cm},
        level 2/.style={sibling distance=3cm},
        level 3/.style={sibling distance=1.5cm},
        every node/.style={draw, rectangle, minimum size=0.8cm, inner sep=2pt},
        edge from parent/.style={draw, thick, -{Latex}}
        ]

        \node (r2) {0}
        child {node {0}
                child {node {0}
                        child {node {0}}
                        child {node {3}}
                    }
                child {node {4}
                        child {node {7}}
                        child {node {0}}
                    }
            }
        child {node {11}
                child {node {11}
                        child {node {11}}
                        child {node {15}}
                    }
                child {node {16}
                        child {node {16}}
                        child {node {22}}
                    }
            };
    \end{tikzpicture}

    \vspace{0.5em}
    \[
        \begin{aligned}
            \text{prescan}[L[v]] & = \text{prescan}[v]                    \\
            \text{prescan}[R[v]] & = \text{sum}[L[v]] + \text{prescan}[v]
        \end{aligned}
    \]
    \caption{\textbf{Associative scan algorithm} (taken from \cite{associative_scan}). The associative scan algorithm computes prefix sums (or any associative operation) in $O(\log S)$ time using parallel computation. First, the up-sweep is performed, where consecutive nodes are combined. Then, the root is set to the identity (zero in the case of the sum). For the down-sweep, the value at a parent node is passed directly to its left child. The value passed to the right child is the parent's value combined with the up-sweep result from its left sibling.}
    \label{fig:associative_scan}
\end{figure}

\begin{definition}[Affine function shorthand]
    We use a shorthand square-bracket notation for affine functions:

    \begin{equation}
        [M, c](x) \coloneq Mx + c.
    \end{equation}
\end{definition}

\begin{definition}[Piping operator]
    Let $f_1$, and $f_2$ be two functions.
    \begin{equation}
        (f_1 \rhd f_2)(x) \coloneq f_2(f_1(x)).
    \end{equation}
\end{definition}

\begin{proposition}[Composition of affine functions]
    \begin{equation}
        [M_1, c_1] \rhd [M_2, c_2] = [M_2 M_1, M_2 c_1 + c_2]
    \end{equation}
\end{proposition}
\begin{proof}
    \begin{align}
        f_2(f_1(x)) = M_2 \left(M_1 x + c_1\right) + c_2 = M_2 M_1 x + M_2 c_1 + c_2 = [M_2 M_1, M_2 c_1 + c_2](x)
    \end{align}
\end{proof}

\begin{theorem}[Affine recurrence lifting]
    Given a recurrence
    \begin{equation}
        h_k = \overline{A} h_{k - 1} + \overline{B} x_k,
    \end{equation}
    with $h_0 = 0$, we can define the sequence of elements (affine functions) given by
    \begin{equation}
        e_k = \left[\overline{A}, \overline{B} x_k\right].
    \end{equation}
    As binary operation, we will take function composition. Note that we want $e_1 \rhd e_2 \rhd \dots \rhd e_n$. Thus, the binary operation can be expressed as
    \begin{equation}
        [M_1, c_1] \bullet [M_2, c_2] = [ M_2 M_1, M_2 c_1 + c_2 ].
    \end{equation}
    The associative scan will produce a sequence of affine functions whose intercepts are $\{h_k\}_k$.
\end{theorem}
\begin{proof}
    The only thing we have to prove is that $\bullet$ associative, that is, that
    \begin{equation}
        (e_1 \bullet e_2) \bullet e_3 = e_1 \bullet (e_2 \bullet e_3).
    \end{equation}
    However, this is equivalent to
    \begin{equation}
        (e_1 \rhd e_2) \rhd e_3 = e_1 \rhd (e_2 \rhd e_3).
    \end{equation}
    which is true because function composition is associative.
\end{proof}

This lets us compute affine recurrences in $O(\log S)$ using multiple compute units. However, each step of the parallel scan involves computing the product of two matrices, which is $O(D^3)$. Once again, we can just assume $\overline{A}$ to be diagonal.

Finally, we should remark that in practice, ordinary scans can be faster than associative scans, as the main bottleneck is often memory overhead \cite{griffin}.

\subsection{Related work}
\subsubsection{S4}
S4 layers \cite{s4} are SSM models that use the convolutional mode for training, and the recurrent view for inference. The matrix $A$ is constrained to be of the form
$
    V \Lambda V^* - P Q^T
$ (normal plus low-rank). They give an algorithm to compute the convolutional kernel $K$ in $O((L + D)(\log(L + D))^2)$. Note that although an associative scan could be used, the matrix $\overline{A}$ not being diagonal would make the algorithm impractically slow.
\subsubsection{SaShiMi}
SaShiMi \cite{SaShiMi} uses S4 layers, but restricts $A = \Lambda - P P^*$. They also introduce a multi-scale architecture, in which they use pooling to reduce the sequence length.
\subsubsection{S5}
For training, S5 layers \cite{s5} use an associative scan. In order to avoid the cubic complexity of matrix multiplication, they make $A$ be diagonal.
\subsubsection{Mamba}
For training, Mamba \cite{mamba} layers use an associative scan. In addition, $B$, $C$ and $\Delta$ are not directly learnable parameters, but depend on the input $x$. The recurrence is given by
\begin{equation}
    h_k = \overline{A_k} h_{k - 1} + \overline{B_k} x_k,
\end{equation}
where $\overline{A_k}$ and $\overline{B_k}$ are functions of $A$ and $x_k$. Note that this can be lifted to an associative scan by defining $e_k = \left[ \overline{A_k}, \overline{B_k} x_k\right]$. Allowing $\overline{A_k}$ and $\overline{B_k}$ to depend on $x$ increases performance, but the time dependency means the convolutional method is no longer valid.
\subsubsection{Griffin}
For training, RG-LRU layers \cite{griffin} can use an associative scan. The recurrence is given by
\begin{equation}
    h_k = a_k \odot h_{k - 1} + \sqrt{1 - a_k^2} \odot i_k \odot x_k,
\end{equation}
where $a_k$, and $i_k$ are derived from $x_k$. Again, note that this can be lifted to an associative scan by defining $e_k = \left(\text{diag}\left(a_k\right), \text{diag}\left(\sqrt{1 - a_k^2} \odot i_k\right)\right)$. Like in Mamba, the convolutional view cannot be applied as $a_k$ depends on $k$.

\chapter{Poolformer}

Modern sequence-to-sequence models are typically based on the transformer architecture, which is organized as a stack of blocks. Each block combines temporal mixing across the sequence dimension (through self-attention layers) and channel mixing (through dense layers). While this design has proven highly effective, the quadratic cost of self-attention with respect to sequence length makes scaling to very long sequences computationally prohibitive.

We aim to design a sequence-to-sequence model suitable for very long sequences. An obvious starting point is replacing self-attention layers with recurrent layers. However, if the sequences are extremely long, this might not be efficient enough. In addition, even though recurrent neural networks have a potentially infinite context window, in practice, they do not. In this thesis, we will study how pooling can help alleviate these issues.

\section{SkipBlocks}

\begin{figure}[h]
    \centering
    \includegraphics[width=0.5\linewidth]{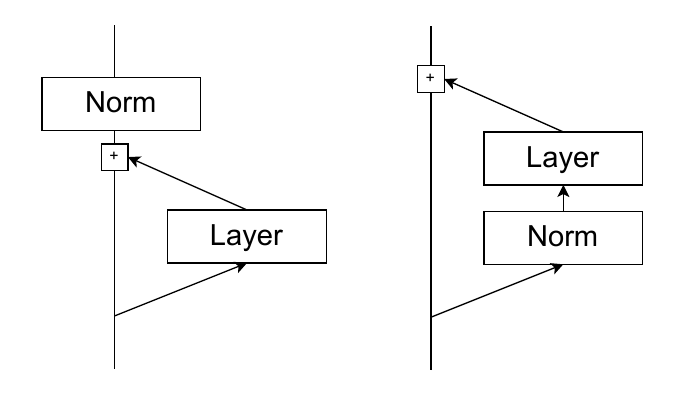}
    \caption{Comparison of post-layer normalization (left) and pre-layer normalization (right).}
    \label{fig:norm}
\end{figure}

The original transformer places the normalization layer between residual blocks (post-layer normalization) \cite{attention}. Pre-layer normalization, on the other hand, moves the normalization layer to the beginning of the residual block. This makes the model more interpretable: the black vertical bar in Figure \ref{fig:norm} (often called the residual stream) is the original input plus the output of the different blocks. At present, pre-layer normalization is more common, as it has been shown to improve training stability (refer to \cite{prelayernorm} for more details).

Normal transformer-like models consist of a stack of residual blocks, each adding its contribution to the residual stream (see Figure \ref{fig:norm}). However, if we apply down-pooling, the dimension of the data changes, and residual connections are no longer feasible. Thus, instead of defining our model as a sequence of blocks, we define it recursively in terms of SkipBlocks.

A SkipBlock consists of a down-pooling layer, an inner SkipBlock, and an up-pooling layer. It also contains MLP and temporal mixing ResBlocks, as a normal transformer would. We consider two ways of introducing skip-connections: from the beginning to the end (long skip-connections) and around the pooling layers (short skip-connections). In Chapter \ref{sec:ablations}, we ablate the importance of having ResBlocks around the pooling layers (instead of just after). See Figure \ref{fig:skipblock} for a visual explanation.

\begin{figure}[h]
    \centering
    \begin{subfigure}{0.48\textwidth}
        \centering
        \includegraphics[width=\linewidth]{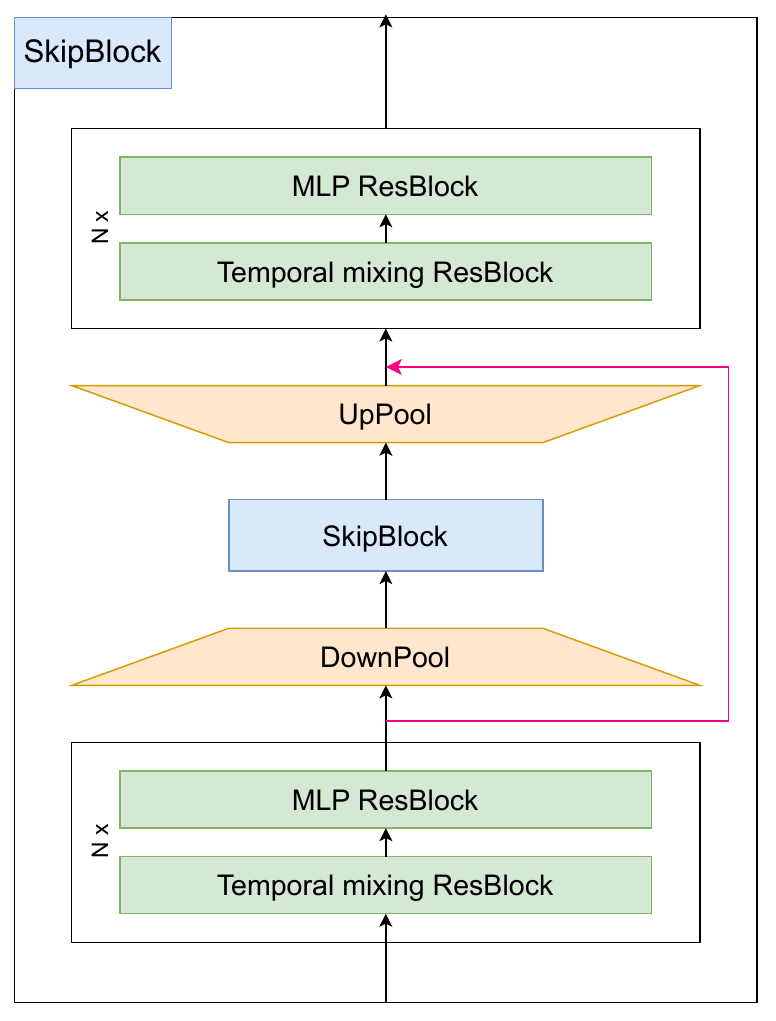}
        \caption{(short skip-connections)}
    \end{subfigure}
    \hfill
    \begin{subfigure}{0.48\textwidth}
        \centering
        \includegraphics[width=\linewidth]{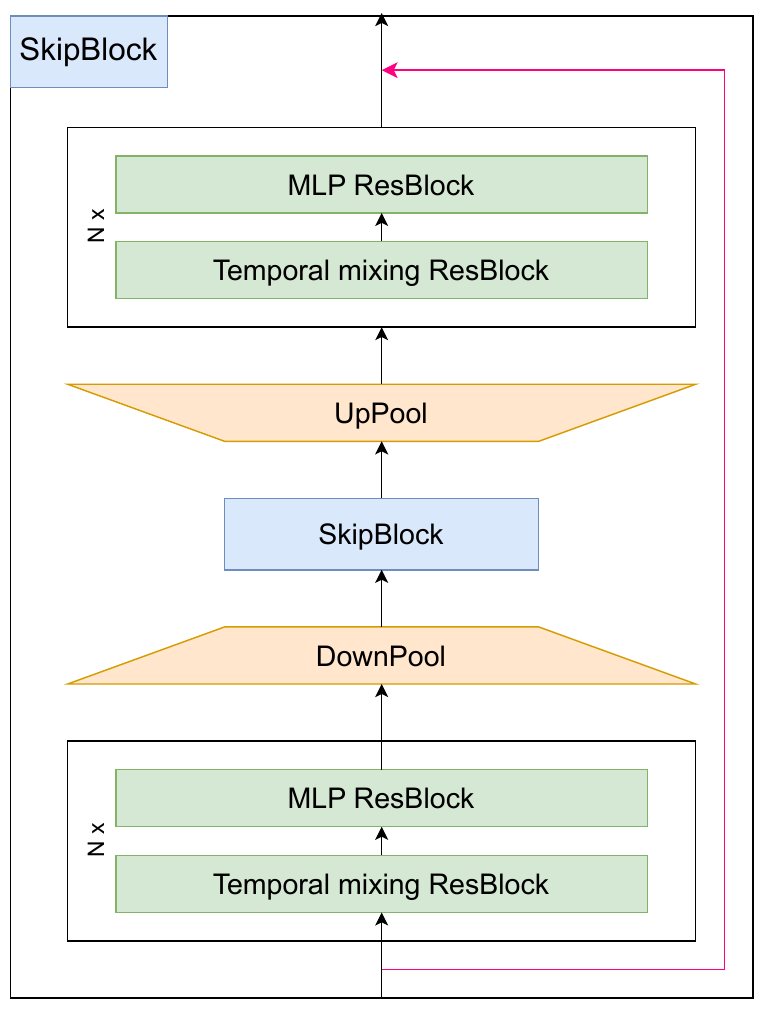}
        \caption{(long skip-connections)}
    \end{subfigure}
    \caption{\textbf{SkipBlock.} A SkipBlock consists of a down-pooling layer, an inner SkipBlock, and an up-pooling layer. It also contains MLP and temporal mixing ResBlocks. The skip-connections (pink) connect distant parts of the residual stream with the same shape to improve gradient flow. }
    \label{fig:skipblock}
\end{figure}

\section{Pooling}
\label{sec:pooling}

\begin{figure}[h]
    \centering
    \begin{subfigure}{0.48\textwidth}
        \centering
        \includegraphics[width=\linewidth]{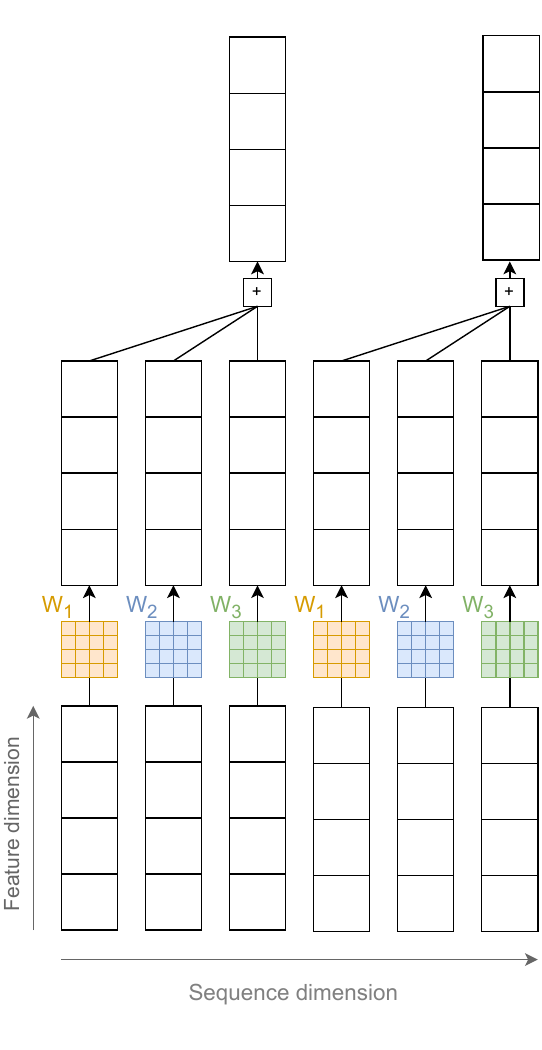}
        \caption{Down-pooling}
    \end{subfigure}
    \hfill
    \begin{subfigure}{0.48\textwidth}
        \centering
        \includegraphics[width=\linewidth]{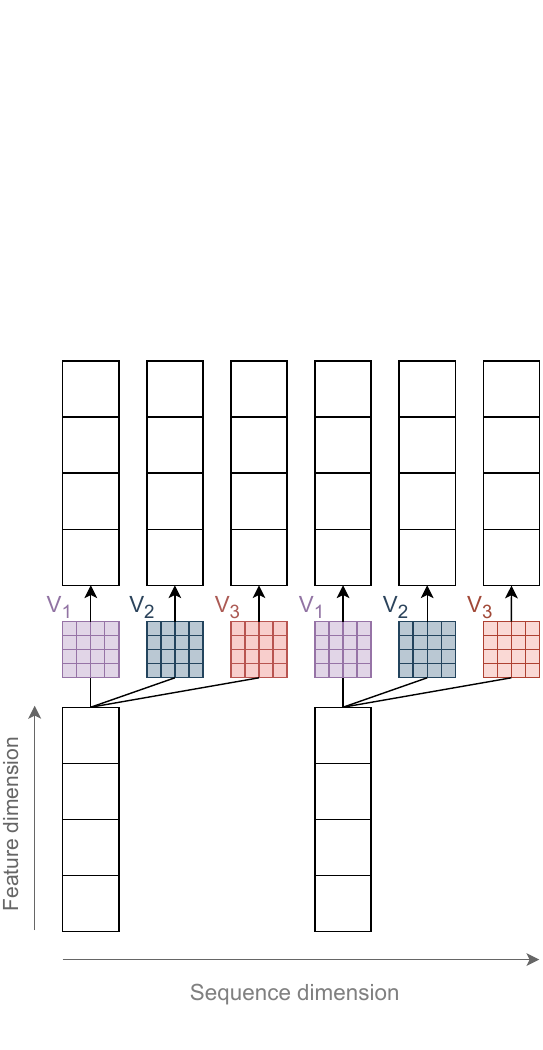}
        \caption{Up-pooling}
    \end{subfigure}
    \caption{\textbf{Pooling.} Pooling is implemented with strided convolutions (down-pooling) and strided transposed convolutions (up-pooling). In this example, the pooling factor is $F = 3$ and the sequence length is 6.}
    \label{fig:pooling_architecture}
\end{figure}

In order to reduce the sequence length by a factor of $F$, we use a 1-dimensional convolution with kernel size $F$, stride $F$, and feature group count $G$. This requires $F D^2 / G$ parameters. The procedure, depicted in Figure \ref{fig:pooling_architecture}, can be summarized as follows:

\begin{enumerate}
    \item Divide the sequence, of shape $\mathbb{R}^{L \times D}$, into $\frac{L}{F}$ disjoint segments of shape $\mathbb{R}^{F \times D}$
    \item Every segment needs to be mapped to a single vector in $\mathbb{R}^D$. In order to do that,
          \begin{enumerate}
              \item We map the $i$-th vector in a segment to a vector in $\mathbb{R}^D$ using $W_i$, where $W_i$ is a block-diagonal matrix with block size equal to $G$. Special cases are $G = 1$ (ordinary matrix) and $G = D$ (diagonal matrix).
              \item We sum all the mapped vectors, from $i = 1$ to $i = F$.
          \end{enumerate}
\end{enumerate}

In order to increase the sequence length by a factor of $F$, we use a transposed convolution with stride $F$. The procedure is:
\begin{enumerate}
    \item For each token $i$ in the down-sampled sequence. Map it with $V_j$, from $j=1$ to $j=F$, where $V_j$ is a block-diagonal matrix with block size equal to $G$.
    \item These will be output tokens $F i$, $F i + 1$, ... $F i + (F - 1)$
\end{enumerate}
The output of the up-pooling layer will be merged onto the residual stream. In order to keep training stable, we initialize the up-pooling layer's weights by sampling from
\begin{equation}
    \mathcal{N}\left(0, \frac{\mathtt{scale}}{n_\mathrm{in}}\right),
\end{equation}
where $n_\mathrm{in}$ is the number of inputs. If $\mathtt{scale} = 1$, we recover Lecun normal initialization. If $\mathtt{scale} = 0$, the behavior of the network at initialization will be identity-like.

\section{ResBlocks}
Both MLPs and temporal mixing layers are wrapped in residual blocks (see Figure \ref{fig:resblock}). A residual block (ResBlock) computes its output as $x + F(x)$, where $x$ is the input and $F(x)$ is a sequence of operations \cite{he2016deep}. This structure helps stabilize training and enables deeper architectures by allowing gradients to flow more easily through the network.

\begin{figure}[h]
    \centering
    \begin{subfigure}{0.48\textwidth}
        \centering
        \includegraphics[width=\linewidth]{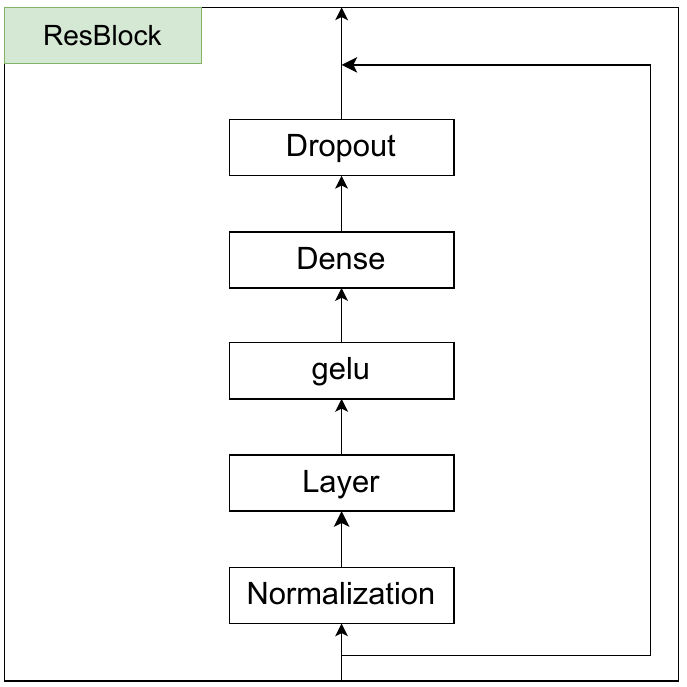}
        \caption{ResBlock without gating}
    \end{subfigure}
    \hfill
    \begin{subfigure}{0.48\textwidth}
        \centering
        \includegraphics[width=\linewidth]{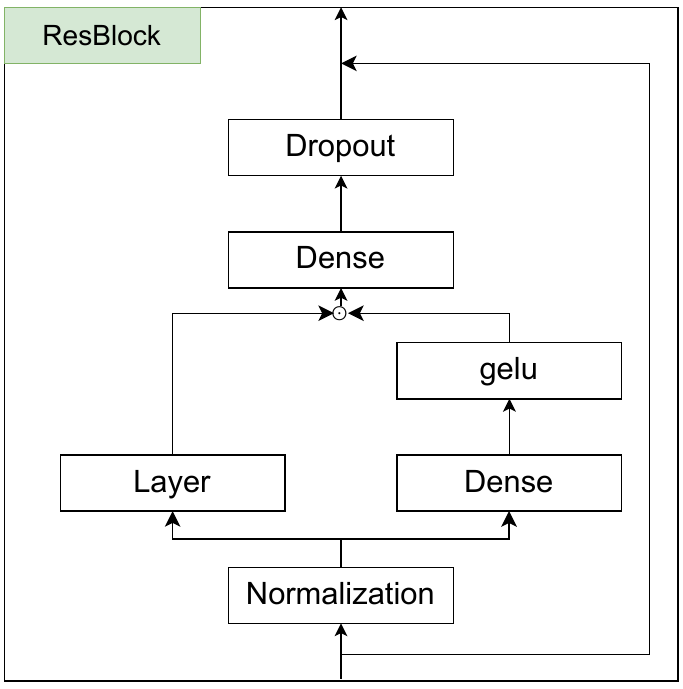}
        \caption{ResBlock with gating}
    \end{subfigure}
    \caption{Comparison of standard ResBlock (left) and gated ResBlock (right).}
    \label{fig:resblock}
\end{figure}

\subsection*{Normalization}

In our residual block, the first operation is a normalization layer. We tried three different normalization strategies:
\begin{itemize}
    \item \textbf{LayerNorm} \cite{layernorm}, in which the tensor is first standardized across the feature dimension, and then scaled and shifted by some learned parameters. That is, given $x \in \mathbb{R}^D$, we compute $\text{LayerNorm}(x)$ independently for each sample and time-step:
          \begin{align}
              \mu                 & = \frac{1}{D} \sum_{i=1}^D x_i                               \\
              \sigma              & = \sqrt{\frac{1}{D} \sum_{i=1}^D (x_i - \mu)^2 + \epsilon}   \\
              \text{LayerNorm}(x) & = \gamma \odot \left( \frac{x - \mu}{\sigma} \right) + \beta
          \end{align}
          where $\gamma, \beta \in \mathbb{R}^D$ are learnable.
    \item \textbf{RMSNorm} \cite{rmsnorm}, in which the tensor is first normalized by the root mean square, and then scaled by some learned parameters. It uses half as many parameters as LayerNorm, and previous research \cite{xlstm7b} found it more stable while training recurrent models.
          \begin{align}
              \text{RMS}(x)     & = \sqrt{\frac{1}{D} \sum_{i=1}^d x_i^2 + \epsilon}    \\
              \text{RMSNorm}(x) & = \gamma \odot \left( \frac{x}{\text{RMS}(x)} \right)
          \end{align}
    \item \textbf{No normalization}, where the normalization layer is replaced by the identity function. This was done to test whether normalization could be avoided with careful initialization.
\end{itemize}

\subsection*{Gating}
Inspired by \cite{griffin}, we consider two variants of our residual block (see Figure \ref{fig:resblock} for a visual explanation):
\begin{itemize}
    \item \textbf{Without gating:} the MLP or temporal mixing layer followed by a gelu non-linearity \cite{gelu}, and then by a dense layer.
    \item \textbf{With gating:} two branches are created, one with the MLP or temporal mixing layer, and another with a dense layer and a gelu. After that, both branches are multiplied component-wise. The result is then passed through a final dense layer.
\end{itemize}
The output of the final dense layer will be passed through a Dropout layer, and then merged onto the residual stream. In order to keep training stable, we initialize the last layer's weights by sampling from
\begin{equation}
    \mathcal{N}\left(0, \frac{\mathtt{scale}}{n_\mathrm{in}}\right),
\end{equation}
where $n_\mathrm{in}$ is the number of inputs. If $\mathtt{scale} = 1$, we recover Lecun normal initialization. If $\mathtt{scale} = 0$, the behavior of the network at initialization will be identity-like.

\subsection*{Dropout}

To avoid overfitting, we include a Dropout layer \cite{dropout} as the last operation. Dropout randomly sets a fraction of the input units to zero during training, which helps prevent the model from relying too heavily on any particular feature and encourages better generalization.

\section{Temporal Mixing Layer}
\label{sec:rg-lru}
As temporal mixing layer, we use the RG-LRU, introduced in \cite{griffin}. Additionally, we explore introducing attention in the deepest layers, where the sequence length is short. The recurrence is given by the following equations.
\begin{align}
    r_k & = \sigma(W_a x_k + b_a)                                         \\
    i_k & = \sigma(W_x x_k + b_x)                                         \\
    a_k & = a^{c r_k}                                                     \\
    h_k & = a_k \odot h_{k - 1} + \sqrt{1 - a_k^2} \odot (i_k \odot x_k), \\
\end{align}
Following \cite{griffin}, we fix $c = 8$. Regarding $a$, we explore using both complex and real values. Note that if the magnitude is small, the hidden states will quickly converge to zero. Similarly, if the magnitude is larger than one, the hidden state will diverge to infinity. Because of this, we use a parametrization that ensures that the magnitude will always be less than one.

\subsubsection{Real case}
In the real case we set $a = \sigma(\Lambda)$, where $\Lambda \in \mathbb{R}^D$. We initialize $\Lambda$ in such a way that $a \sim \sqrt{\mathcal{U}[0.9^2, 0.99^2]}$.

\subsubsection{Complex case}

In the complex case we set $a = \sigma(\Lambda) \exp(i \theta)$, where $\Lambda \in \mathbb{R}^\frac{D}{2}$, and $\theta \in \mathbb{R}^\frac{D}{2}$. We initialize $\Lambda$ in such a way that $\sigma(\Lambda)$ follows a $\sqrt{\mathcal{U}[0.9^2, 0.99^2]}$. On the other hand, $\theta$ is initialized by sampling from a $\mathcal{U}[0, \frac{\pi}{10}]$, following \cite{orvieto2023resurrecting}. This results in a uniform initialization on the short ring in the complex plane centered at the origin, with radius in the interval $[0.9,0.99]$ and angle restricted to the interval $[0,\frac{\pi}{10}]$. Note that if we had sampled $\sigma(\Lambda)$ from  $\mathcal{U}[0.9, 0.99]$, the distribution wouldn't be uniform on the ring, as the concentration of points would be larger closer to the origin.

Even if the recursion occurs in complex space, we need to map the output back to real numbers to use ordinary layers, such as dense layers. In order to do this, we concatenate the real and imaginary parts, obtaining a vector in $\mathbb{R}^D$.

\chapter{Experiments}
\section{Datasets}
\label{sec:datasets}
We use the same three datasets from \cite{SaShiMi}, keeping the same train / validation / test split. All of them are audio recordings sampled at 16000 Hz with 16-bit linear PCM encoding.
\begin{enumerate}
    \item \textbf{SC09} \cite{sc09,warden2018speech}, which consists of 1-second (16,000 tokens) recordings of the spoken digits zero to nine (similar to MNIST \cite{mnist}, but with audio instead of images). For this dataset, we also compute the FID and IS scores using SaShiMi's repo \cite{SaShiMi}.
    \item \textbf{Beethoven} \cite{beethoven}, that contains 8-second (128,000 tokens) recordings of Beethoven's piano sonatas.
    \item \textbf{YouTubeMix} \cite{youtubemix}, that contains 1-minute (960,512 tokens) recordings of Beethoven's piano sonatas.
\end{enumerate}
Following \cite{SaShiMi}, we apply $\mu$-law encoding to SC09 and YouTubeMix, and linear encoding to Beethoven, which results in unsigned 8-bit representations.

\section{Experimental setup}
We use JAX and related libraries (JAX, FLAX, Optax) \cite{deepmind2020jax}, and train all of our models on TPU v4-8 and v3-128. We use the AdamW optimizer \cite{adamw} with a learning rate of 0.002. All other hyperparameters are set to the default values in Optax: $\beta_1 = 0.9$, $\beta_2 = 0.999$, and a weight decay of 0.0001. Training is performed with:
\begin{itemize}
    \item A constant learning rate schedule, using a linear warmup for the first 1000 steps.
    \item A batch size of 32.
    \item 500 total training epochs.
\end{itemize}

Model evaluation is conducted on the validation and test sets at the end of every epoch. We apply an exponential moving average (EMA) to the model weights during evaluation, with an EMA decay rate of 0.999.

\begin{definition}[Pooling configuration]
    The pooling configuration is specified by a list of integers. For example, our baseline model has a pooling configuration of $[2, 4, 4, 5]$, which means that in the deepest level, the sequence has $2 \cdot 4 \cdot 4 \cdot 5 = 160$ times fewer tokens than in the shallowest level.
\end{definition}

\begin{definition}[Layer configuration]
    The layer configuration defines the number of RG-LRU layers at each level of the model's hierarchy. It is specified by a list of $n+1$ integers, $[l_1,l_2,...,l_n,l_{n+1}]$, for a model with $n$ SkipBlocks.
    \begin{itemize}
        \item For each of the $n$ SkipBlocks, the corresponding integer $l_i$ specifies the number of RG-LRU layers applied both before down-pooling and after up-pooling. Thus, the $i$-th SkipBlock contains a total of $2 \cdot l_i$ layers.
        \item The final integer, $l_{n+1}$, specifies the number of layers in the deepest part of the network.
    \end{itemize}
    For example, our baseline layer configuration is $[4,4,4,4,4]$. In this case, $n=4$:
    \begin{itemize}
        \item There are $4$ SkipBlocks. Each is defined by the value $4$, meaning each contains $2 \cdot 4=8$ RG-LRU layers ($4$ layers before down-pooling and $4$ layers after up-pooling).
        \item The deepest level is defined by the final value of $4$, meaning it contains $4$ RG-LRU layers.
        \item The total number of RG-LRU layers is therefore $4 \cdot 8 + 4 = 36$.
    \end{itemize}
\end{definition}

In terms of dimensionality, we distinguish between the RG-LRU dimension, which applies to the variables $r_k$, $i_k$, $a_k$, and $h_k$ within the temporal mixing layer (see Section \ref{sec:rg-lru}), and the model dimension, which applies to all other components. Our baseline has a model dimension of 128, a dropout rate of 0.2, and RG-LRU layers with complex numbers and a hidden dimension of 256, which brings the total number of parameters to 7,272,704.

\section{Signal pre-processing and EMA}
\subsection*{Signal pre-processing}
The original audio files are 16-bit, but we convert them into 8-bit. Because of this, the input to the Poolformer are numbers between 0 and 255, which have to be mapped to $D$-dimensional vectors. This can be done in different ways:
\begin{itemize}
    \item \textbf{Sinusoidal embeddings}. Similarly to the positional embeddings of \cite{attention}, or the Fourier features of \cite{kingma2021variational}, we can map the class $i$ (number from 0 to 255) to the concatenation of $c$ and $s$
          \begin{align}
              c_{ij} & = \cos\left(\frac{i}{10000^\frac{j}{D / 2}}\right)  \\
              s_{ij} & = \sin\left(\frac{i}{10000^\frac{j}{D / 2}}\right),
          \end{align}
          where $j$ goes from 0 to $D / 2 - 1$. See Figure \ref{fig:sinusoidal} for a more visual explanation.
    \item \textbf{Linear scaling followed by a dense layer}. We can first make the input lie between 0 and 1, and then apply a dense layer. This preserves the natural ordering, but makes similar numbers almost indistinguishable.
    \item \textbf{Embedding layer}, so that each class is mapped to a learned feature vector.
\end{itemize}

As shown in Figure \ref{fig:preprocessing}, sinusoidal embeddings provide the model with a very meaningful representation of the input, and results in the best performance. Because of this, sinusoidal embeddings are used in all of our experiments. Note that when using an embedding layer, the performance starts off being significantly worse, but improves over time as the model learns better and better representations. In the long run, both give comparable results.

\begin{figure}[h]
    \centering
    \begin{subfigure}{0.48\textwidth}
        \centering
        \includegraphics[width=\linewidth]{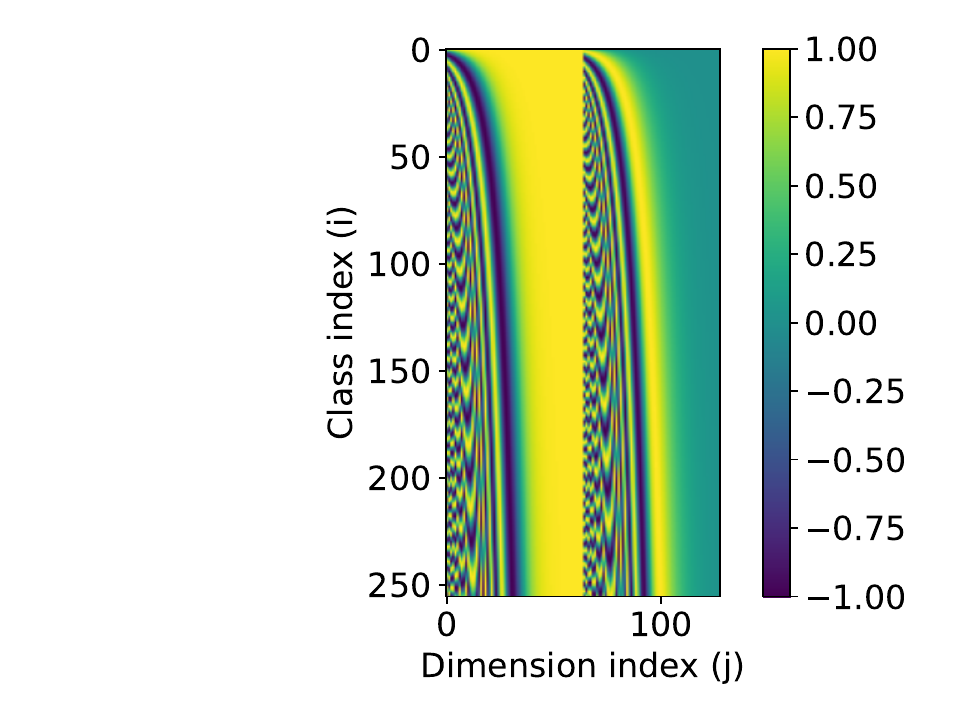}
        \caption{Sinusoidal embeddings}
    \end{subfigure}
    \hfill
    \begin{subfigure}{0.48\textwidth}
        \centering
        \includegraphics[width=\linewidth]{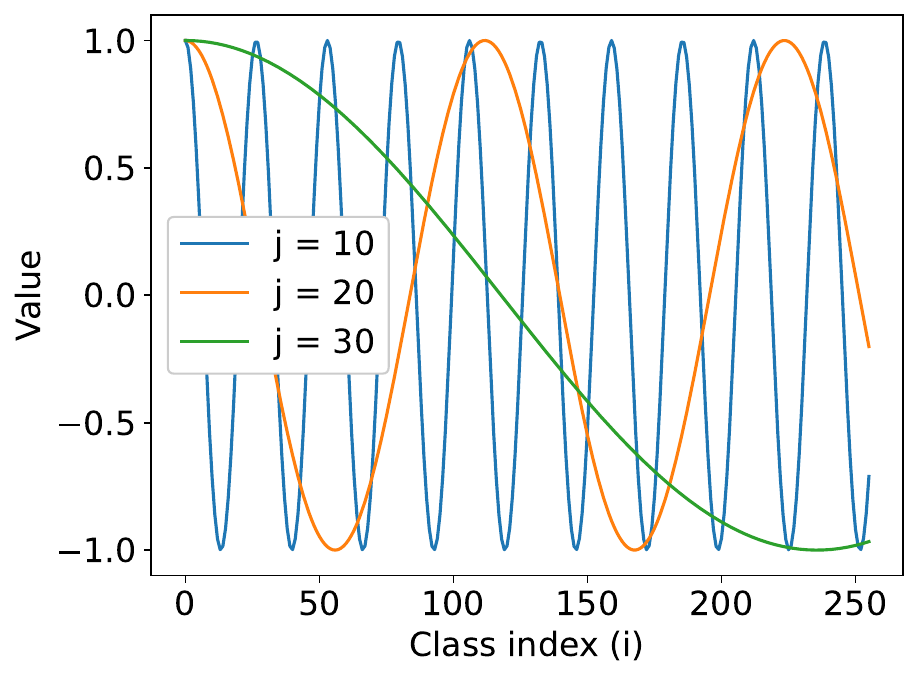}
        \caption{Some slices from the embeddings}
    \end{subfigure}
    \caption{\textbf{Sinusoidal embeddings.} Each embedding column is a cosine (first half) or a sine (second half). As $j$ increases, the frequency of the wave decreases.}
    \label{fig:sinusoidal}
\end{figure}

\begin{figure}[h]
    \centering
    \begin{subfigure}{0.5\textwidth}
        \centering
        \includegraphics[width=\linewidth]{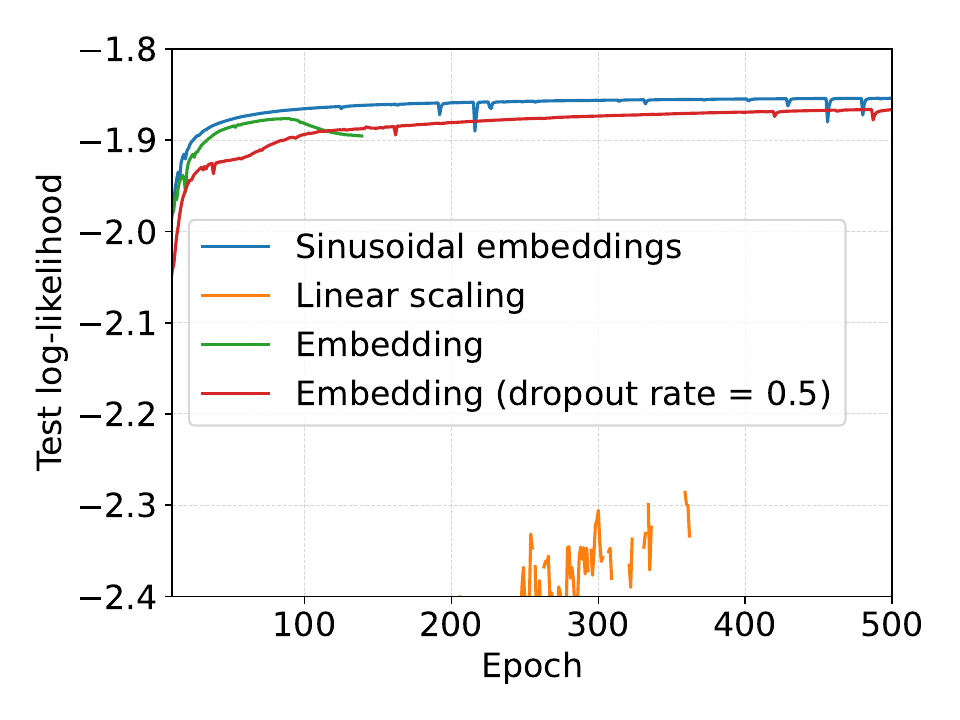}
    \end{subfigure}
    \caption{\textbf{Pre-processing comparison} (SC09). The linear scaling strategy performs very poorly, and the model cannot really distinguish between similar inputs. Using embeddings seems to result in overfitting, which can be solved by increasing the dropout rate. Overall, sinusoidal embeddings yield the best performance. }
    \label{fig:preprocessing}
\end{figure}

\subsection*{Exponential Moving Average}
Exponential moving average (EMA) is a technique commonly used to stabilize and improve model evaluation. During training, EMA maintains a running average of the model parameters, where recent updates are weighted more heavily than older ones:

\begin{equation}
    w_\text{EMA} \xleftarrow{} \alpha w_\text{EMA} + (1 - \alpha) w_\text{new}.
\end{equation}

At evaluation time, predictions are made using these averaged parameters rather than the most recent ones. This helps to smooth out fluctuations caused by noisy updates and often leads to better generalization and more consistent validation or test performance. Typical values for $\alpha$ are 0.999 (our baseline), and 0 (no EMA). In Figure \ref{fig:ema} we can see that keeping an EMA of the weights is beneficial, as it results in a smoother and improved test curve.

\begin{figure}[h]
    \centering
    \includegraphics[width=0.6\linewidth]{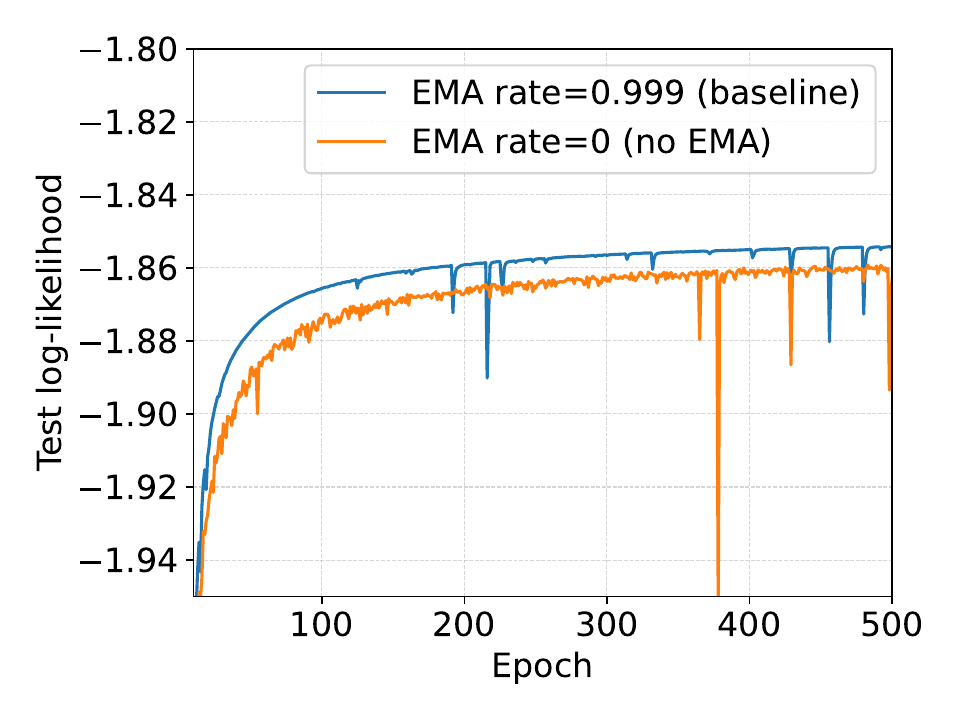}
    \caption{\textbf{Effect of exponential moving average (EMA)} (SC09). Applying EMA to model weights during evaluation leads to smoother and slightly improved test curves.}
    \label{fig:ema}
\end{figure}

\section{Model Ablations}
\label{sec:ablations}
In this section, we compare different models to our baseline. We perform these ablations on the SC09 dataset exclusively. All results are summarized in Table \ref{tab:ablations}.

\begin{table}[h]
    \centering
    \begin{tabular}{lcccc}
        \hline
        \textbf{Model}                          & \textbf{Num. params.} & \textbf{Test NLL $\downarrow$} & \textbf{Notes} \\
        \hline
        Poolformer (baseline)                   & 7,272,704             & 1.854 (1.854)                  &                \\
        \quad w/ RMSNorm instead of LayerNorm   & 7,263,488             & 1.874 (1.859)                  & Unstable       \\
        \quad w/ $\mathtt{scale} = 0$           & 7,272,704             & 1.852 (1.852)                  &                \\
        \quad w/ $\mathtt{scale} = 1$           & 7,272,704             & 1.886 (1.874)                  & Overfits       \\
        \quad w/o gating                        & 5,489,408             & 2.264 (1.859)                  & Very unstable  \\
        \quad w/ long-skip connections          & 7,272,704             & 1.895 (1.895)                  &                \\
        \quad all ResBlocks after pooling       & 7,272,704             & \textbf{1.851 (1.848)}         &                \\
        \quad w/ attention in the deepest layer & 7,229,184             & 1.852 (1.851)                  &                \\
        \quad w/o complex numbers               & 7,355,648             & 1.859 (1.859)                  &                \\
        \quad w/ half ring initialization       & 7,272,704             & 1.854 (1.854)                  &                \\
        \quad w/ full ring initialization       & 7,272,704             & 1.876 (1.861)                  & Very unstable  \\
        \quad w/ full circle initialization     & 7,272,704             & 1.998 (1.924)                  & Very unstable  \\
        \hline
        SaShiMi                                 & 4,751,096             & 1.891                                           \\
        Mamba                                   & 6,100,000             & 1.852                                           \\
        Mamba                                   & 24,300,000            & 1.860                                           \\
        \hline
    \end{tabular}
    \caption{\textbf{Ablation results} (SC09). For Poolformer, we report the final negative log-likelihood on the test set, with the best result (corresponding to the lowest validation loss) shown in parentheses.}
    \label{tab:ablations}
\end{table}

\subsection*{Normalization}
We were not able to train the model without a normalization layer, even with careful initialization. Unlike \cite{xlstm7b}, we find LayerNorm more beneficial than RMSNorm regarding model stability (see Figure \ref{fig:norm-comparison}).

\begin{figure}[h]
    \centering
    \begin{subfigure}{0.48\textwidth}
        \centering
        \includegraphics[width=\linewidth]{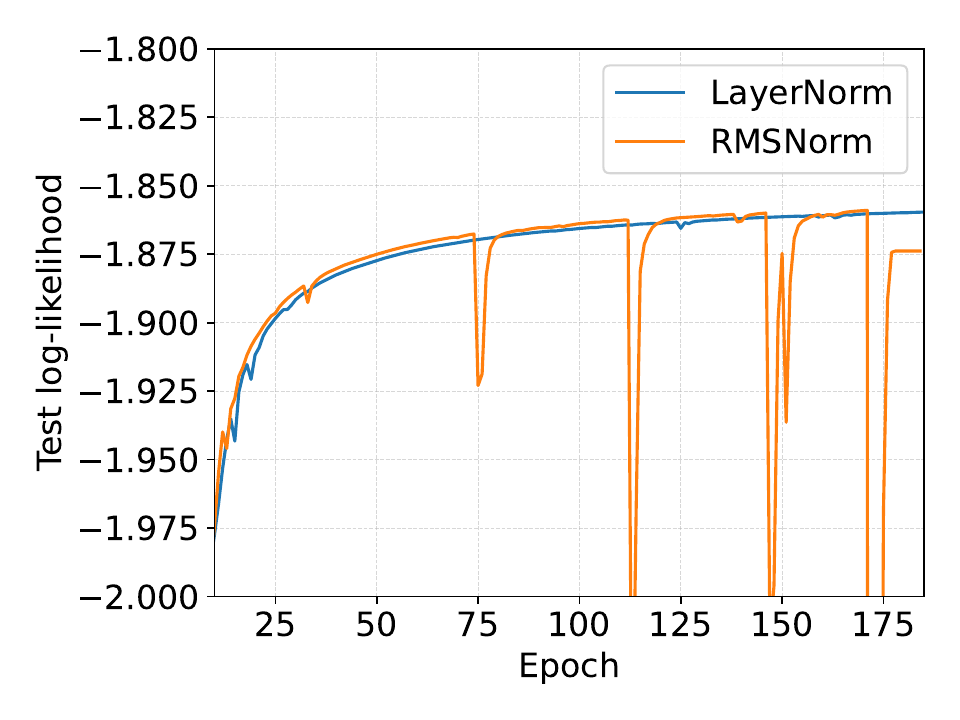}
    \end{subfigure}
    \hfill
    \begin{subfigure}{0.48\textwidth}
        \centering
        \includegraphics[width=\linewidth]{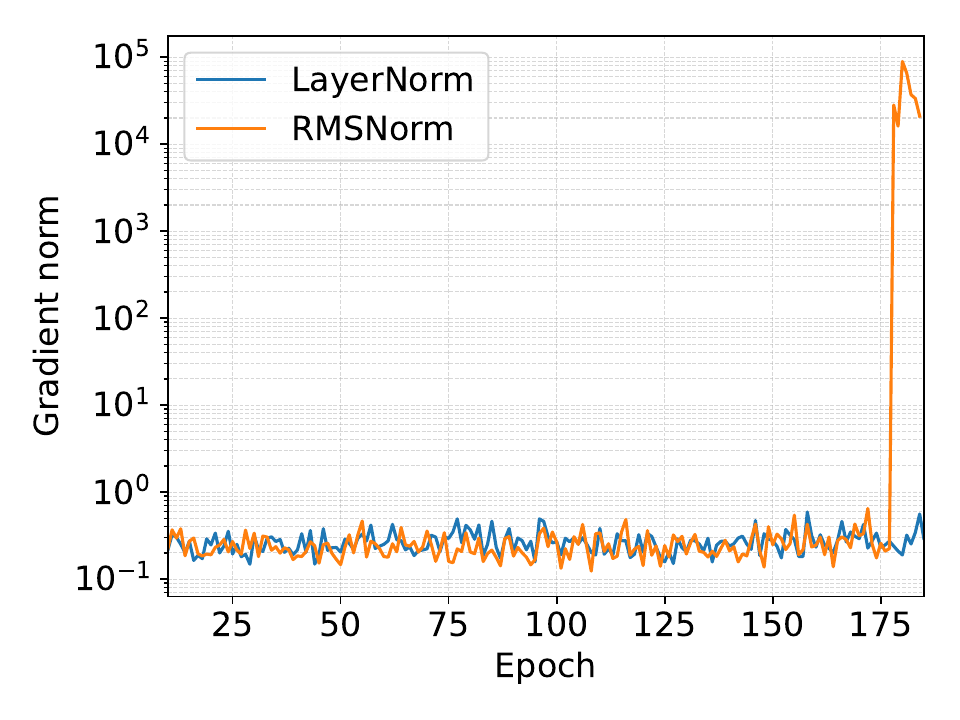}
    \end{subfigure}
    \caption{\textbf{Comparison of LayerNorm and RMSNorm} (SC09). We find that LayerNorm makes training significantly more stable.}
    \label{fig:norm-comparison}
\end{figure}

\subsection*{Initialization}
In our SkipBlock model, the residual stream gets contributions from the input, the residual blocks, and the up-pooling layer. In order to keep training stable, we initialize the last layers of the residual blocks and up-pooling layers sampling from a
\begin{equation}
    \mathcal{N}\left(0, \frac{\mathtt{scale}}{n_\mathrm{in}}\right),
\end{equation}
where $n_\mathrm{in}$ is the number of inputs. We try $\mathtt{scale} = 0$ (zero initialization), $\mathtt{scale} = \frac{1}{n_{blocks}}$ (our baseline), and $\mathtt{scale} = 1$ (Lecun normal initialization). The results are shown in Figure \ref{fig:init}: setting $\mathtt{scale} \ll 1$ is crucial for getting good performance.

\begin{figure}[h]
    \centering
    \includegraphics[width=0.6\linewidth]{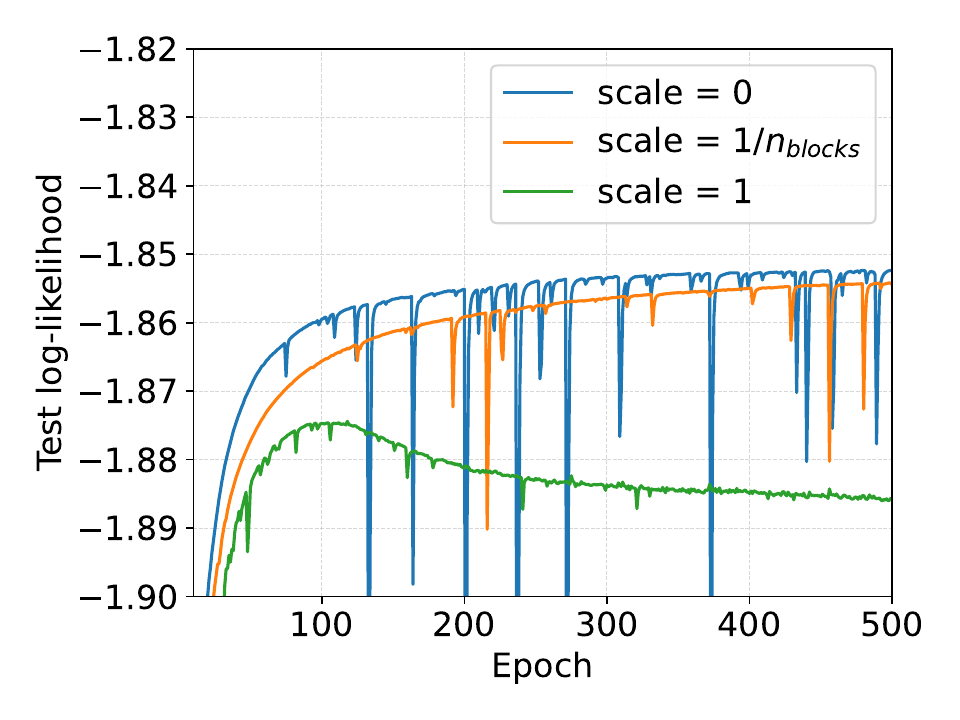}
    \caption{\textbf{Effect of initialization scaling} (SC09). Effect of scaling down the initialization variance of those layers whose outputs are added to the residual stream. Early in training, having a small scale is beneficial for stability, as the model has an identity-like behavior. In the long run, it seems to affect stability negatively. Interestingly, we see that setting $\mathtt{scale} = 1$ results in overfitting. }
    \label{fig:init}
\end{figure}

\subsection*{Gating}
The gated variant of our residual block creates two branches, one with the MLP or temporal mixing layer, and another one with a dense layer (see Figure \ref{fig:resblock}). After this, both branches are multiplied component-wise and passed through a final dense layer. This results in significantly more parameters (7,272,704 instead of 5,489,408). As we can see in Figure \ref{fig:gating} the performance gains were quite modest. However, the stability  greatly improves with gating.

\begin{figure}[h]
    \centering
    \begin{subfigure}{0.48\textwidth}
        \centering
        \includegraphics[width=\linewidth]{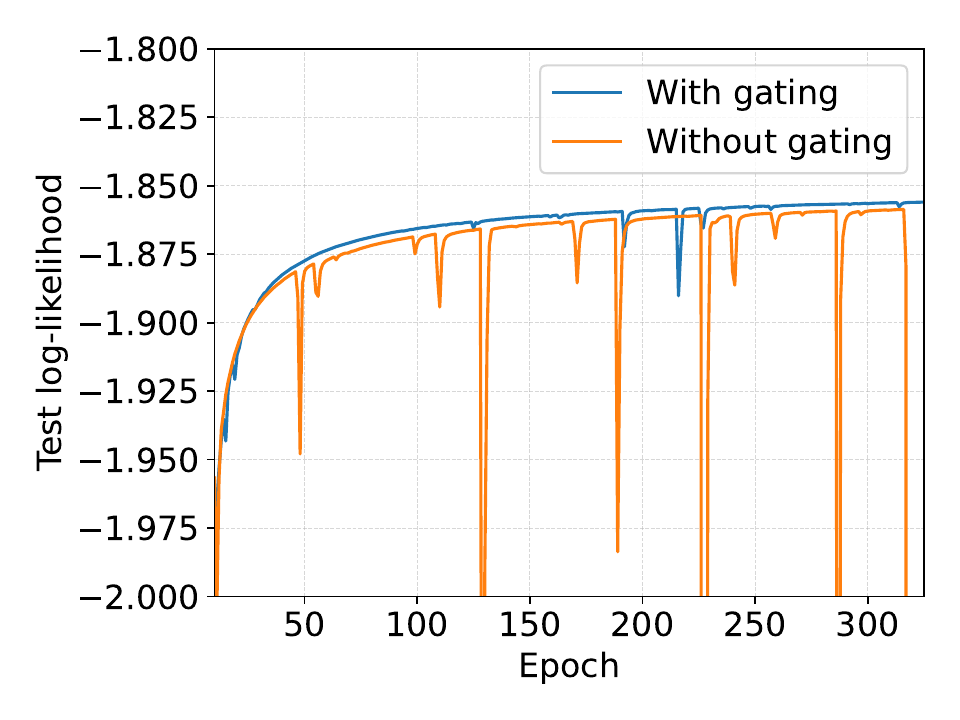}
    \end{subfigure}
    \hfill
    \begin{subfigure}{0.48\textwidth}
        \centering
        \includegraphics[width=\linewidth]{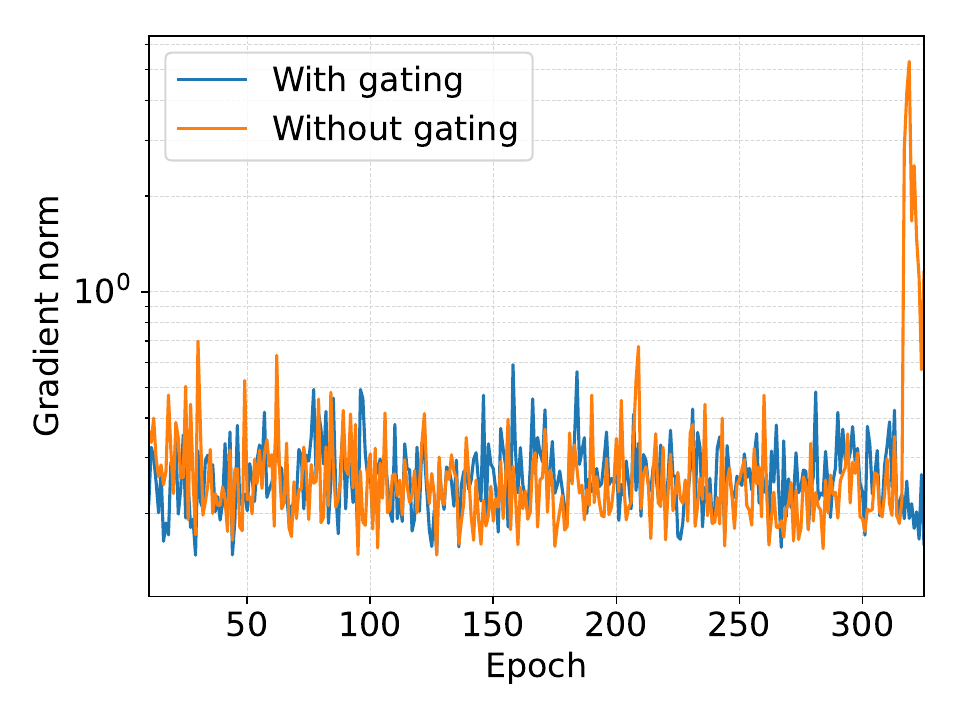}
    \end{subfigure}
    \caption{\textbf{Effect of gating} (SC09). We found gating to be very beneficial for training stability. }
    \label{fig:gating}
\end{figure}

\subsection*{Skip-connection}
As we discussed in Figure \ref{fig:skipblock}, there are two main ways of adding skip-connections in our model. Our experiments show that the best way is to add the residual connection around the pooling layers (short skip-connections).

This makes sense if we consider the alternative (long skip-connections), in which the skip-connection is added after the residual blocks. In the outer-most layer, the input would be copied directly to the output, which means that the final residual blocks would have to first cancel out the features from the original input before adding their own contribution.

\subsection*{ResBlocks position}
We can place residual blocks before or after the pooling layers (see Figure \ref{fig:skipblock}). In this experiment, we consider two alternatives: having residual blocks both before and after, and having residual blocks just after the pooling layers. To make the comparison fair, we make the total number of RG-LRU layers the same.

We show that placing all of them after the pooling results in a slight increase in performance: the test negative log-likelihood goes down from 1.854 to 1.851. We hypothesize that the reason for this is that most of the short-term processing is performed in the last layers, so placing 8 layers instead of 4 after the last up-pooling layer results in an increase in performance. See Figure \ref{fig:magnitude-by-depth} for more information.

\subsection*{Attention in the deepest layer}
In this experiment, we replace the four RG-LRU layers in the deepest level with multi-query attention (4 heads). The reason to place the self-attention layer there is that the sequence length is just $\frac{16000}{2 \cdot 4 \cdot 4 \cdot 5} = 100$, so it is computationally feasible. This results in a slight increase in performance.

\subsection*{Complex numbers}
The $a$ parameter in RG-LRU layers can be real or complex. Recent research has shown that the complex parametrization is not useful for language modeling \cite{griffin}. In our case, it seems that for audio modeling, the complex parametrization slightly outperforms the real one.

As explained in Section \ref{sec:rg-lru}, $a$ is parametrized to lie within the unit circle in the complex plane, and we expect most of the values to be close to the boundary. It is interesting to consider several initialization schemes
\begin{itemize}
    \item \textbf{Small ring (baseline):} sample the magnitude from $\sqrt{\mathcal{U}[0.9^2, 0.99^2]}$, and the phase from $\mathcal{U}[0, \frac{\pi}{10}]$, as in \cite{orvieto2023resurrecting}.
    \item \textbf{Half ring:} sample the magnitude from $\sqrt{\mathcal{U}[0.9^2, 0.99^2]}$, and the phase from $\mathcal{U}[0, \pi]$.
    \item \textbf{Full ring} sample the magnitude from $\sqrt{\mathcal{U}[0.9^2, 0.99^2]}$, and the phase from $\mathcal{U}[0, 2\pi]$.
    \item \textbf{Full circle} sample the magnitude from $\sqrt{\mathcal{U}[0.01^2, 0.99^2]}$, and the phase from $\mathcal{U}[0, 2\pi]$.
\end{itemize}
In Figure \ref{fig:complex_numbers} we can visually see the $a$ values at initialization, and after training. In Table \ref{tab:ablations} we observe similar results as \cite{orvieto2023resurrecting}: reducing the initialization range of the phase seems crucial for capturing long-range dependencies.

\begin{figure}[h]
    \centering
    \begin{subfigure}{0.35\textwidth}
        \centering
        \includegraphics[width=\linewidth]{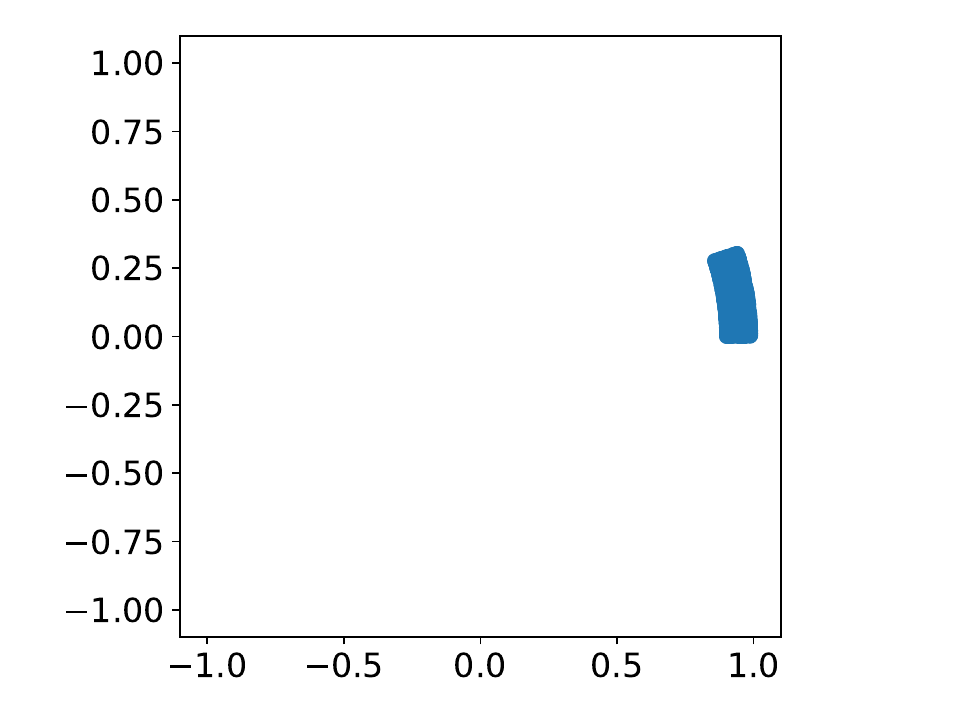}
        \caption{Small ring (baseline) after initialization}
    \end{subfigure}
    \hfill
    \begin{subfigure}{0.35\textwidth}
        \centering
        \includegraphics[width=\linewidth]{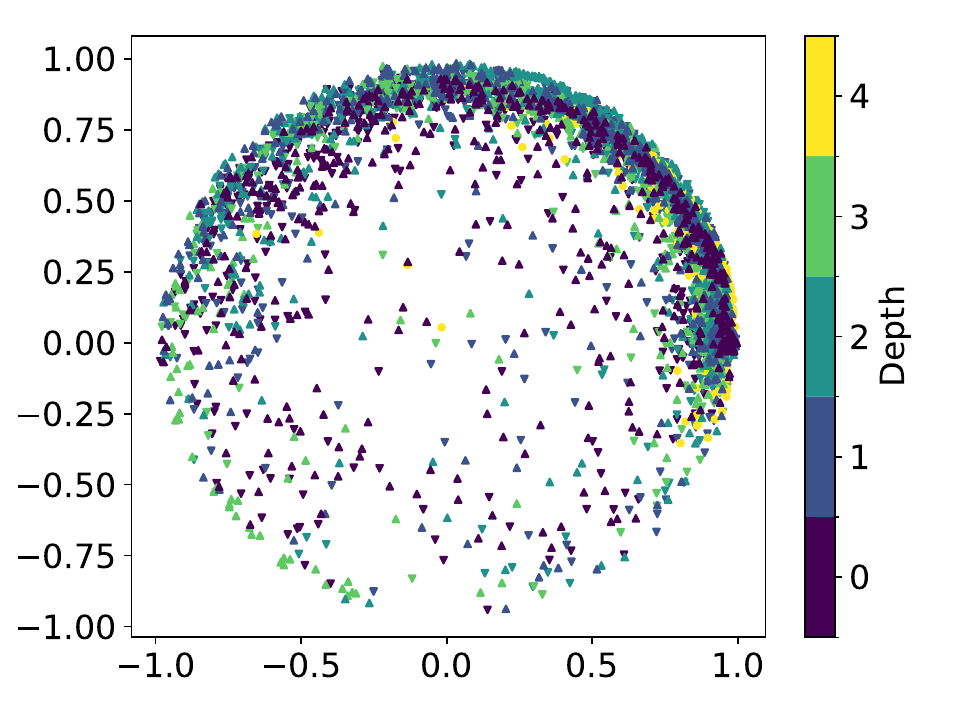}
        \caption{Small ring (baseline) after training}
    \end{subfigure}
    \vskip\baselineskip
    \begin{subfigure}{0.35\textwidth}
        \centering
        \includegraphics[width=\linewidth]{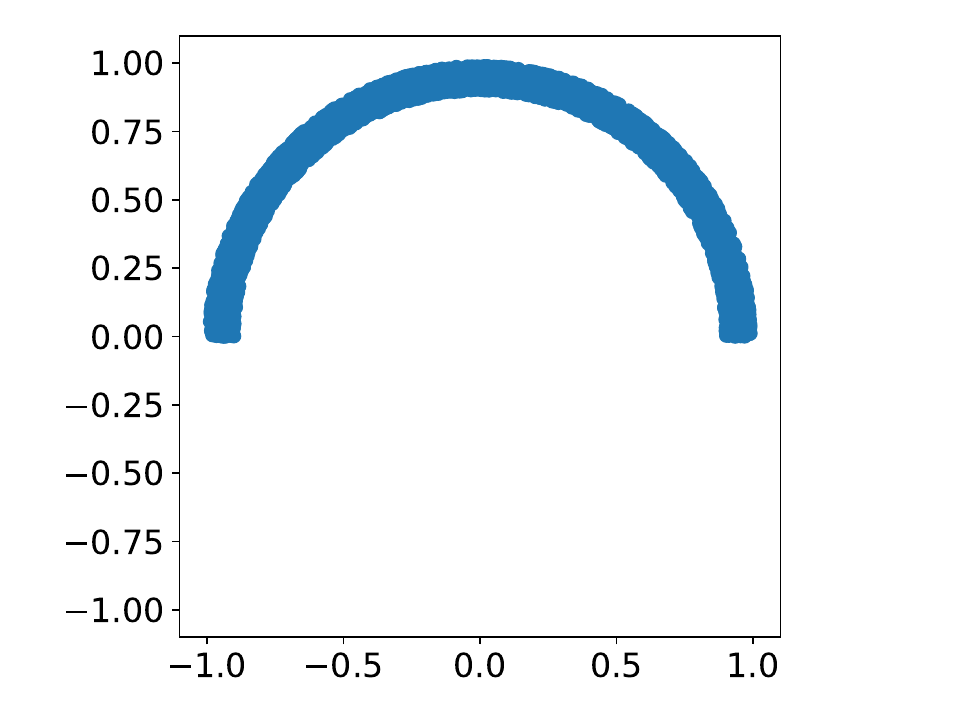}
        \caption{Half ring after initialization}
    \end{subfigure}
    \hfill
    \begin{subfigure}{0.35\textwidth}
        \centering
        \includegraphics[width=\linewidth]{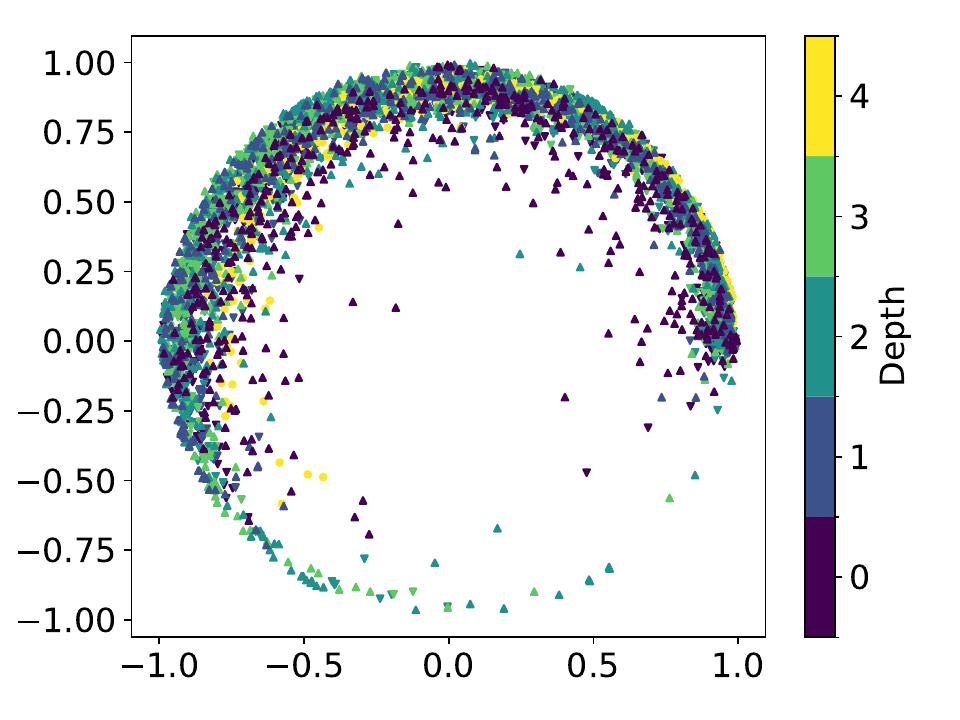}
        \caption{Half ring after training}
    \end{subfigure}
    \vskip\baselineskip
    \begin{subfigure}{0.35\textwidth}
        \centering
        \includegraphics[width=\linewidth]{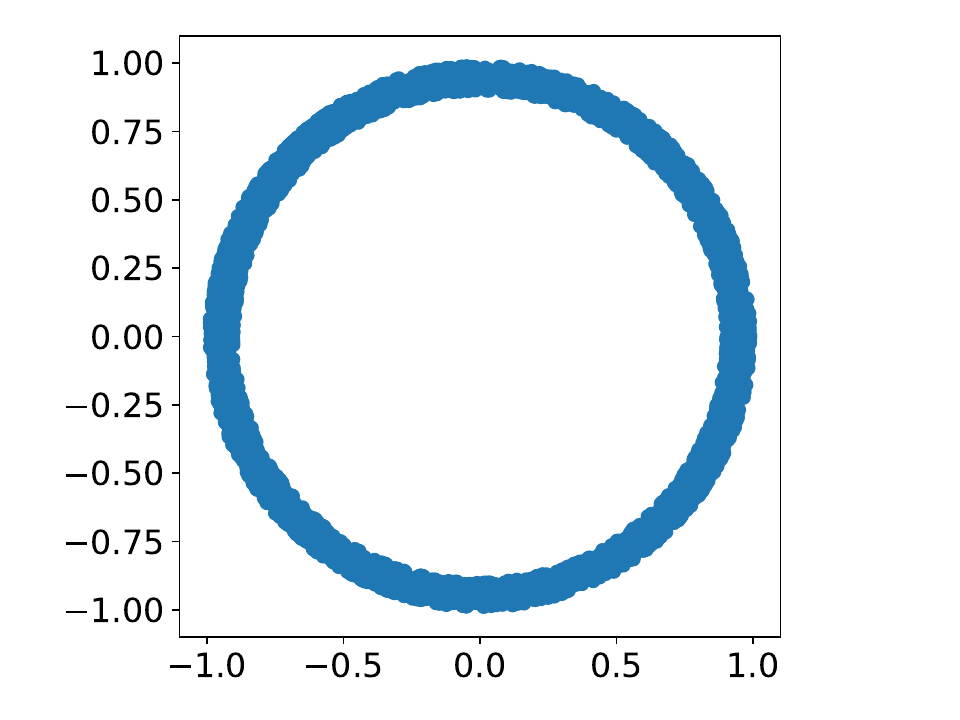}
        \caption{Full ring after initialization}
    \end{subfigure}
    \hfill
    \begin{subfigure}{0.35\textwidth}
        \centering
        \includegraphics[width=\linewidth]{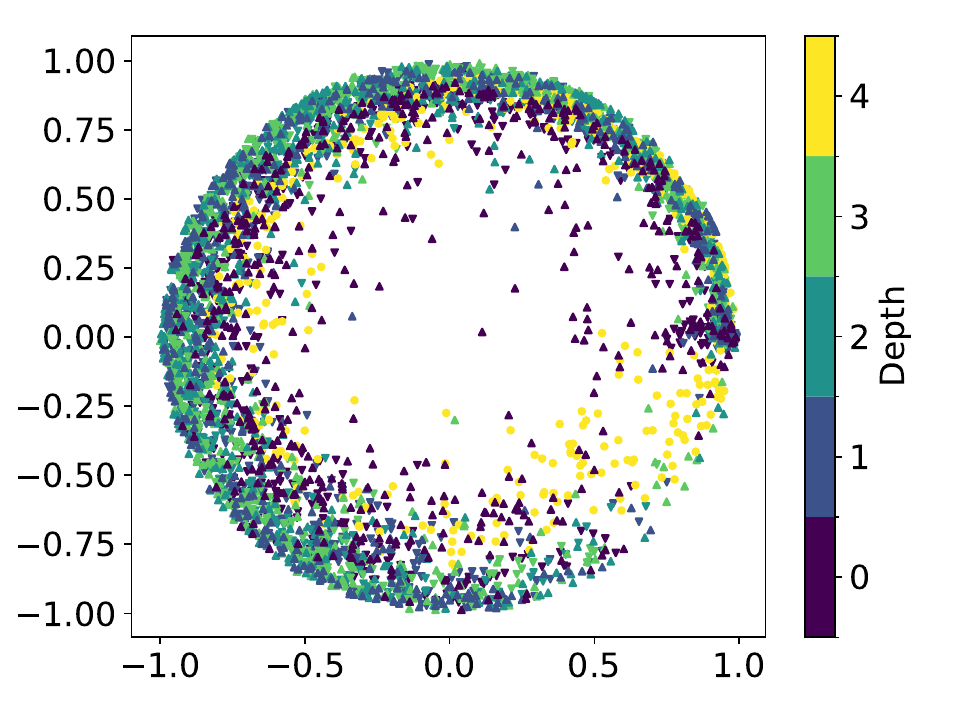}
        \caption{Full ring after training}
    \end{subfigure}
    \vskip\baselineskip
    \begin{subfigure}{0.35\textwidth}
        \centering
        \includegraphics[width=\linewidth]{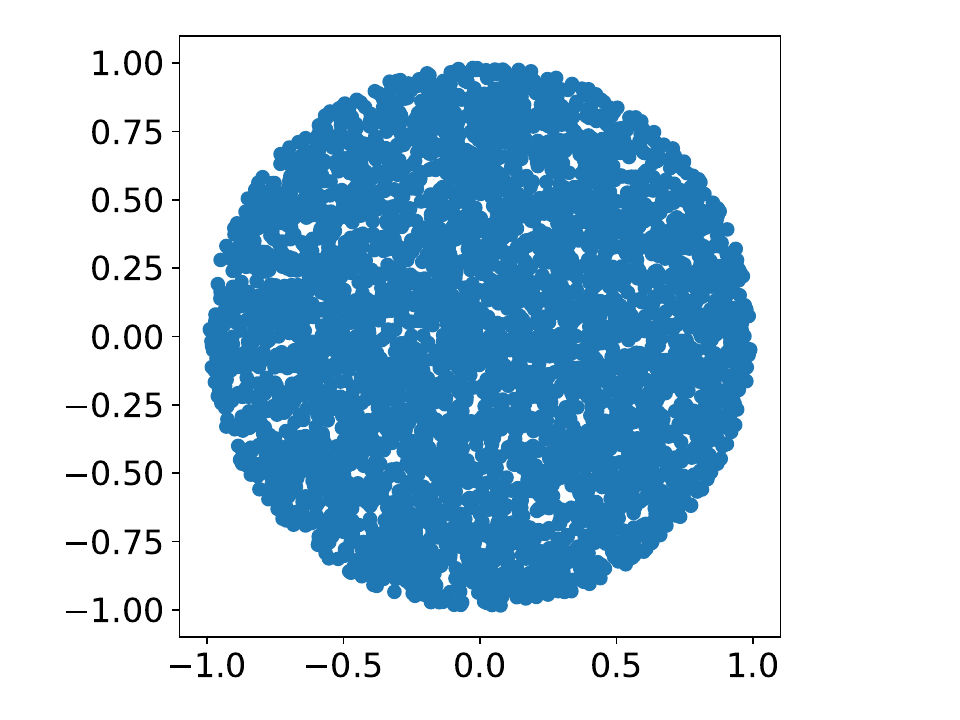}
        \caption{Full circle after initialization}
    \end{subfigure}
    \hfill
    \begin{subfigure}{0.35\textwidth}
        \centering
        \includegraphics[width=\linewidth]{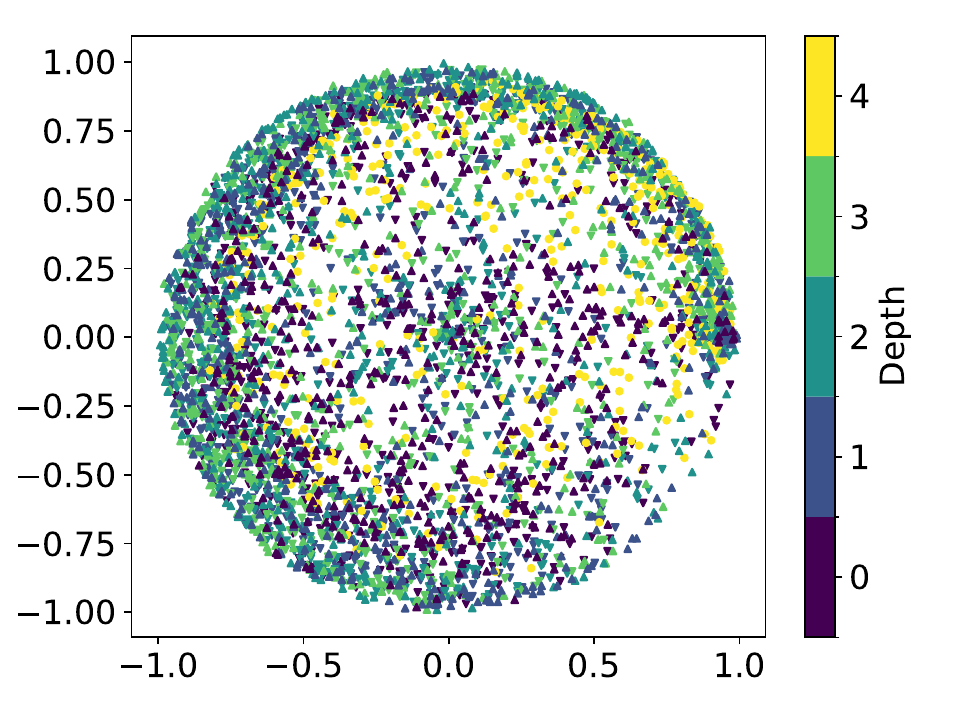}
        \caption{Full circle after training}
    \end{subfigure}
    \caption{\textbf{Complex number initialization} (SC09). Visualization in the complex plane of all the $a$ values in our baseline model after initialization and after training.}
    \label{fig:complex_numbers}
\end{figure}

\section{Pooling Experiments}
In this section, we use our baseline model on the SC09 dataset.

\subsection*{Effects of Pooling in Performance}
We compare our baseline against models with the same number of RG-LRU layers and a similar number of parameters (7.2 million). The results of the experiments can be seen in Table \ref{tab:pooling_importance} and Figure \ref{fig:pooling-overfitting}: pooling greatly accelerates training, improves perceptual metrics (FID and IS), and prevents overfitting. The perceptual metrics align with a qualitative analysis: samples from our baseline sound noticeably better than those from models without that much pooling.

\begin{table}[h]
    \centering

    \begin{minipage}{\linewidth}
        \centering
        \textit{(a) Model configurations}
        \vspace{0.5em}

        \begin{tabular}{lccc}
            \toprule
            \textbf{Model}  & \textbf{Pooling conf.} & \textbf{Layer conf.} & \textbf{Params} \\
            \midrule
            Baseline        & [2, 4, 4, 5]           & [4, 4, 4, 4, 4]      & 7,272,704       \\
            1 pooling layer & [2]                    & [12, 12]             & 7,268,608       \\
            No pooling      & []                     & [36]                 & 7,267,840       \\
            \bottomrule
        \end{tabular}
    \end{minipage}

    \vspace{1em}

    \begin{minipage}{\linewidth}
        \centering
        \textit{(b) Training performance and evaluation metrics}
        \vspace{0.5em}

        \begin{tabular}{lcccc}
            \toprule
            \textbf{Model}  & \textbf{Training speed $\uparrow$} & \textbf{Test NLL $\downarrow$} & \textbf{FID $\downarrow$} & \textbf{IS $\uparrow$} \\
            \midrule
            Baseline        & \textbf{11.45 epochs/h}            & \textbf{1.854} (1.854)         & \textbf{0.46}             & \textbf{6.46}          \\
            1 pooling layer & 6.69 epochs/h                      & 1.882 (1.853)                  & 6.85                      & 2.03                   \\
            No pooling      & 6.11 epochs/h                      & 1.865 \textbf{(1.852)}         & 2.95                      & 3.33                   \\
            \bottomrule
        \end{tabular}
    \end{minipage}

    \caption{\textbf{Pooling importance} (SC09). All models had 36 RG-LRU layers and were trained on a v3-128 TPU pod. Using pooling greatly accelerates training, improves perceptual metrics (FID and IS), and prevents overfitting. We report final negative log-likelihood (NLL) on the test set, with the best result (lowest validation loss) in parentheses. The samples used to compute the FID and IS scores were generated using the best model according to the validation loss.}
    \label{tab:pooling_importance}
\end{table}

\begin{figure}[h]
    \centering
    \includegraphics[width=0.6\linewidth]{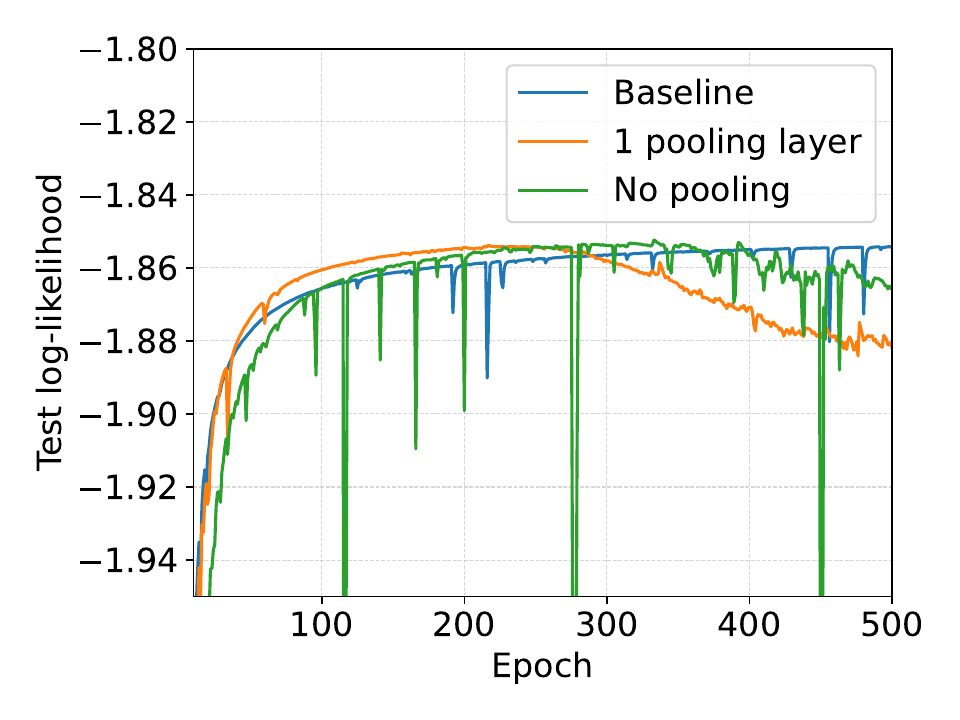}
    \caption{\textbf{Effect of pooling on overfitting} (SC09). Pooling reduces overfitting, as shown by the test log-likelihood curve.}
    \label{fig:pooling-overfitting}
\end{figure}

\subsection*{Effect of $G$}
As explained in Section \ref{sec:pooling}, the $G$ parameter controls the expressivity of the weight matrices within the pooling layers. For example, if $G = 1$ there are no constraints on the weight matrices, and if $G = D$ the weight matrices are diagonal. The higher the value of $G$, the lower the expressivity and the number of parameters. In Table \ref{tab:pooling_g} we see that setting $G = 1$ results in very unstable training.

\begin{table}[h]
    \centering
    \begin{tabular}{lccc}
        \hline
        \textbf{Pooling group count $G$} & \textbf{Parameters} & \textbf{Test NLL $\downarrow$} & \textbf{Notes} \\
        \hline
        1 (full)                         & 7,760,384           & 2.903 (1.956)                  & Very unstable  \\
        2                                & 7,514,624           & 1.858 (1.849)                  &                \\
        4                                & 7,391,744           & \textbf{1.848 (1.848)}         &                \\
        8                                & 7,330,304           & 1.853 (1.853)                  &                \\
        $D$ (diagonal, baseline)         & 7,272,704           & 1.854 (1.854)                  &                \\
        \hline
    \end{tabular}
    \caption{\textbf{Effect of pooling group count $G$ on model performance} (SC09). We report the final negative log-likelihood on the test set, with the best result (corresponding to the lowest validation loss) shown in parentheses.}
    \label{tab:pooling_g}
\end{table}

\subsection*{Depth-wise Analysis}

A central motivation for using pooling is to enable better modeling of long-range dependencies. In our case, we investigate this behavior through the lens of the complex parameter \( a \), which governs how information is retained across timesteps in the RG-LRU.

The parameter \( a \in \mathbb{C}^\frac{D}{2} \) is defined as \(a = \sigma(\Lambda) \cdot e^{i\theta}\), where \( \Lambda \in \mathbb{R}^\frac{D}{2} \) and \( \theta \in \mathbb{R}^\frac{D}{2} \) are learned parameters. The magnitude \( |a| = \sigma(\Lambda) \in \mathbb{R}^\frac{D}{2} \) determines how quickly the recurrent state decays over time: smaller values cause faster decay (short-term memory), while values close to 1 allow information to persist longer (long-term memory).

We use the average magnitude of \( a \) at each depth as a proxy for the temporal range that each layer is sensitive to. Specifically, values of \( |a| \approx 1 \) indicate layers that are better equipped to model long-range dependencies, whereas layers with \( |a| \ll 1 \) tend to focus on short-term patterns.

In our experiments, it seems that deeper layers handle long-term dependencies, while shallower layers (especially the last ones) focus on short-term features. See Figure \ref{fig:magnitude-by-depth} for more information.

\begin{figure}[h]
    \centering
    \includegraphics[width=0.6\linewidth]{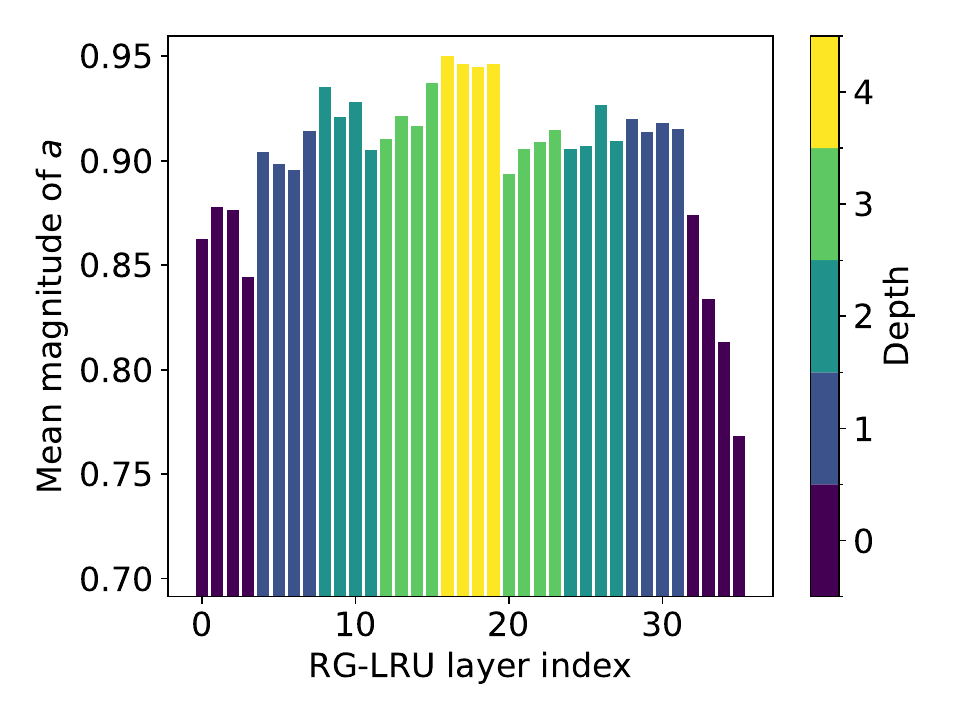}
    \caption{\textbf{Mean magnitude of \( a \) across model depth} (baseline model, SC09). This plot shows how the average magnitude \( |a| \) evolves across the layers of the model. Deeper layers exhibit values closer to 1, indicating slower decay and a bias toward long-range dependencies. In contrast, shallower layers (especially at the end) show smaller magnitudes, suggesting they focus more on short-term features.}
    \label{fig:magnitude-by-depth}
\end{figure}

\section{Scaling Study}

Recurrent neural networks are forced to compress the context into a fixed-size hidden state, which makes them efficient but limits their effectiveness. Thus, a larger hidden state should make the model more effective but slower \cite{mamba}.

In this experiment, we adopt a pooling configuration of $[2, 4]$, a layer configuration of $[4, 4, 4]$, and modify the dimension of the recurrence. In Figure \ref{fig:model_size} and Table \ref{tab:scaling} we see that increasing the recurrence dimension improves performance. Increasing the model dimension significantly raises the parameter count,
which leads to overfitting.

\begin{table}[h]
    \centering
    \begin{tabular}{lcccccl}
        \hline
        \textbf{Model} & \textbf{Model dim.} & \textbf{Recurrence dim.} & \textbf{Num. params.} & \textbf{Test NLL $\downarrow$} \\
        \hline
        Poolformer     & 128                 & 256                      & 4,054,528             & 1.858 (1.853)                  \\
        Poolformer     & 256                 & 256                      & 13,295,360            & 1.908 (1.866)                  \\
        Poolformer     & 384                 & 256                      & 27,779,072            & 1.917 (1.862)                  \\
        Poolformer     & 128                 & 512                      & 5,503,488             & \textbf{1.848 (1.847)}         \\
        Poolformer     & 128                 & 1024                     & 8,647,168             & 1.850 (1.849)                  \\
        \hline
        SaShiMi        & 64, 128, 256        & 64, 128, 256             & 4,751,096             & 1.891                          \\
        Mamba          & 64, 128, 256        & 64, 128, 256             & 6,100,000             & 1.852                          \\
        Mamba          & 96, 184, 368        & 96, 184, 368             & 24,300,000            & 1.860                          \\
        \hline
    \end{tabular}
    \caption{\textbf{Scaling study} (SC09). Here, Poolformer uses a pooling configuration of $[2, 4]$ and a layer configuration of $[4, 4, 4]$. Increasing the recurrence dimension improves performance, whereas increasing the model dimension significantly raises the parameter count, leading to overfitting. We report the final negative log-likelihood on the test set, with the best result (corresponding to the lowest validation loss) shown in parentheses.}
    \label{tab:scaling}
\end{table}

\begin{figure}[h]
    \centering
    \begin{subfigure}{0.48\textwidth}
        \centering
        \includegraphics[width=\linewidth]{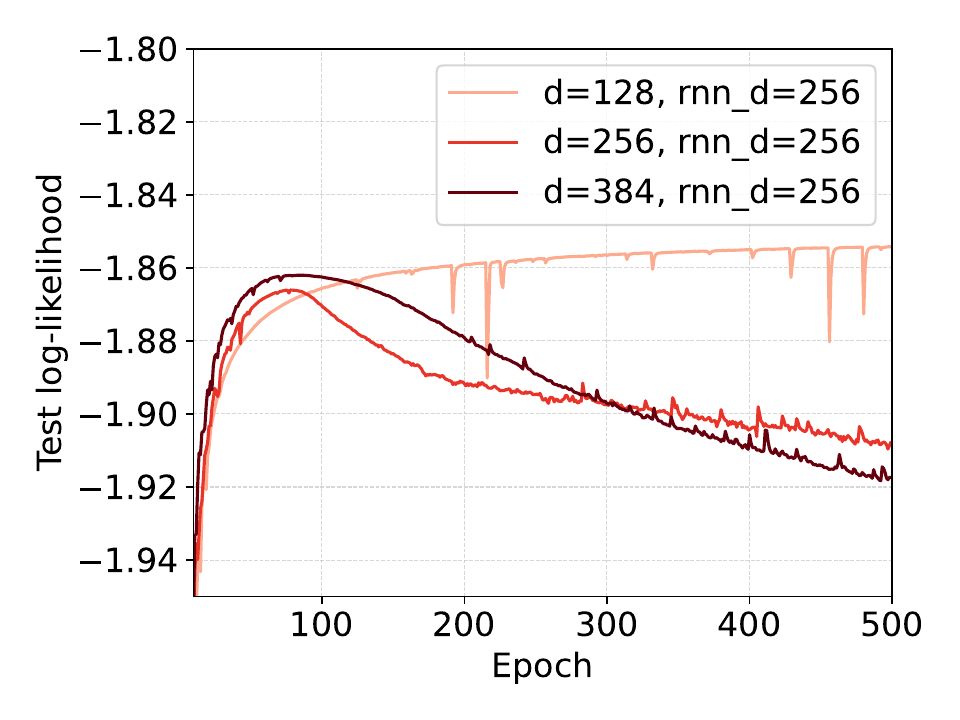}
        \caption{Effect of the model dimension}
    \end{subfigure}
    \hfill
    \begin{subfigure}{0.48\textwidth}
        \centering
        \includegraphics[width=\linewidth]{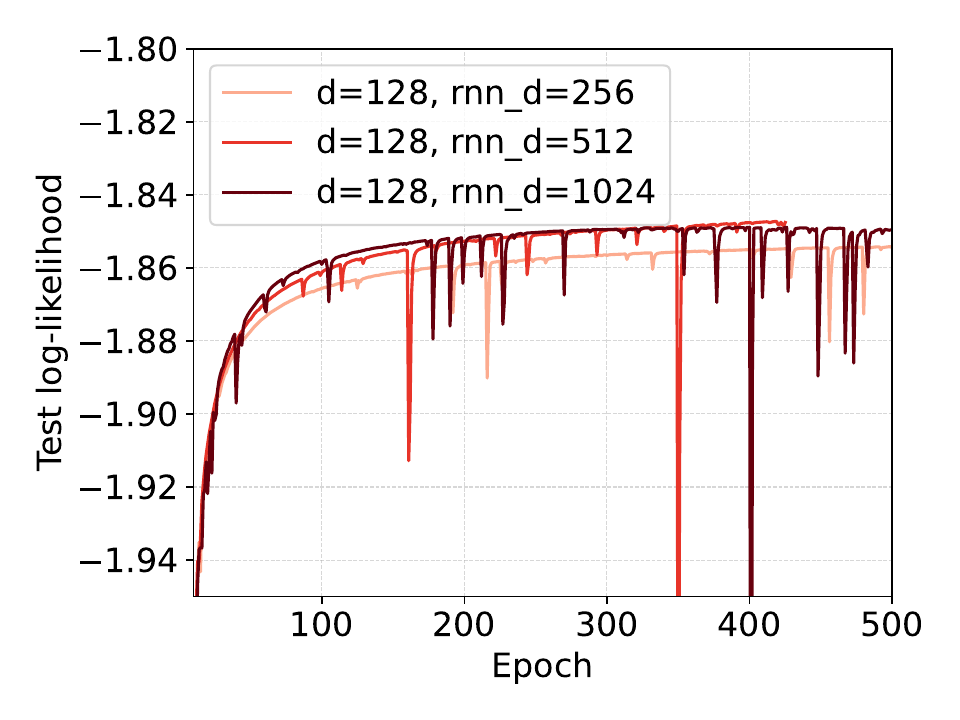}
        \caption{Effect of the RG-LRU dimension}
    \end{subfigure}
    \caption{\textbf{Effect of model size} (SC09). Increasing the model dimension results in overfitting, while increasing the recurrence dimension improves performance.}
    \label{fig:model_size}
\end{figure}

\section{Comparison with the State of the Art}
Table \ref{tab:sota_sc09_comparison} presents a comparison between Poolformer and state-of-the-art models such as SaShiMi and Mamba on the SC09 dataset. Poolformer achieves notably strong performance on perceptual metrics like FID and IS, which we attribute to its ability to capture long-range dependencies through pooling. In terms of negative log-likelihood, our baseline performs comparably to Mamba. Notably, several Poolformer variants, such as those with all ResBlocks after pooling, attention at the deepest layer, or architectural modifications like decreased group count ($G=4$) or increased dimensionality ($D=512$ and $D=1024$), surpass Mamba's NLL figure.

\begin{table}[h]
    \centering
    \begin{tabular}{lccccc}
        \hline
        \textbf{Model}             & \textbf{Test NLL $\downarrow$} & \textbf{FID $\downarrow$} & \textbf{IS $\uparrow$} & \textbf{mIS $\uparrow$} & \textbf{AM $\downarrow$} \\
        \hline
        Poolformer baseline (ours) & 1.854                          & \textbf{0.46}             & \textbf{6.46}          & \textbf{102.10}         & \textbf{0.45}            \\
        SaShiMi                    & 1.891                          & 1.81                      & 3.89                   & 21.90                   & 0.95                     \\
        Mamba                      & \textbf{1.852}                 & 0.94                      & 6.26                   & 88.54                   & 0.52                     \\
        \hline
    \end{tabular}
    \caption{\textbf{Comparison with state-of-the-art models} (SC09).} We don't use rejection sampling to compute the perceptual metrics.
    \label{tab:sota_sc09_comparison}
\end{table}

Table \ref{tab:sota_comparison} compares the negative log-likelihood of these models on the Beethoven and YouTubeMix datasets. For the Beethoven dataset we used a pooling configuration of $[2,4,4,5]$, a layer configuration of $[8,4,4,4,4]$, and a recurrence dimension of 512. We train the model for 66K steps, and got a negative log-likelihood of 0.915, which is significantly better than SaShiMi, which was trained for 1M steps.

For YouTubeMix, we used a pooling configuration of $[2, 2, 4, 4]$, a layer configuration of $[8, 4, 4, 4, 4]$, and a recurrence dimension of 512. We trained the model for 10K steps and obtained a negative log-likelihood of 1.391. While this is worse than SaShiMi's result, it's worth noting that their model was trained for 600K steps. Additionally, our model operates on full 1-minute sequences, whereas SaShiMi is trained on shorter 8-second chunks.

\begin{table}[h]
    \centering
    \begin{tabular}{lcc}
        \hline
        \textbf{Model}    & \textbf{Beethoven Test NLL $\downarrow$} & \textbf{YouTubeMix Test NLL $\downarrow$} \\
        \hline
        Poolformer (ours) & \textbf{0.915}                           & 1.391                                     \\
        SaShiMi           & 0.946                                    & \textbf{1.294}                            \\
        \hline
    \end{tabular}
    \caption{\textbf{Comparison with state-of-the-art models} (Beethoven and YouTubeMix). Our YouTubeMix model was trained on full 1-minute sequences for 10K steps, while SaShiMi was trained on 8-second chunks for 600K steps.}
    \label{tab:sota_comparison}
\end{table}

\chapter{Conclusions, Limitations, and Future Work}
\section{Conclusions}
We have introduced Poolformer, an autoregressive sequence-to-sequence model suitable for extremely long sequences. It is based on the transformer architecture, but replaces self-attention with recurrent layers. In addition, it uses pooling to decrease the sequence length, while keeping skip connections for appropriate gradient flow. Inference is implemented sequentially, while training can be done both sequentially and with a parallel scan.

In our ablations, we show that several techniques are crucial for training stability: normalization layers, careful initialization of those layers feeding into the residual stream, gating, and restricting the expressivity of the pooling layers.

In our pooling experiments, we show that pooling greatly accelerates training, improves perceptual metrics (FID and IS), and prevents overfitting. Studying the magnitude of the $a$ parameter, we speculate that long-range dependencies are handled in the deepest layers (where the sequence is shortest), and that shallower layers take care of short-term features (especially toward the final layers).

Our scaling study shows that increasing the hidden dimension of the recurrent layers results in better performance. This is to be expected, as they must compress all context into a fixed-size state, so a larger dimension helps. Increasing the model dimension significantly raises the parameter count, leading to overfitting.

Finally, we show that our model beats state-of-the-art models such as SaShiMi and Mamba in terms of log-likelihood, and especially in terms of perceptual metrics such as FID and IS. We attribute this to Poolformer's ability to capture long-range dependencies through pooling.

\section{Limitations}
We did not attempt to prevent overfitting in our experiments. However, in practical applications, data is often abundant, making overfitting less of a concern. Additionally, data augmentation techniques could be employed to further mitigate overfitting.

Another limitation is that we only evaluated Poolformer on raw audio data. While this domain naturally features long sequences, further work is needed to assess the model's effectiveness on other modalities.

\section{Future Work}
We evaluated Poolformer exclusively on raw audio data, demonstrating its capacity to handle long sequences. However, we believe the potential applications of Poolformer extend far beyond this setting. In particular, future research could explore its use on other data modalities such as text, images, and video. These domains also suffer from sequence-length challenges, especially when detailed, fine-grained information must be retained over extended contexts.

One particularly promising direction lies in multi-modal learning. For example, vision-language models (VLMs) often use CLIP to provide image embeddings that capture rich semantic content. However, these embeddings often abstract away local details critical for fine-grained reasoning. An alternative is to encode an image into hundreds of patch-level tokens using a pre-trained vision transformer, and pass the dense patch-level representation directly to a language model. Processing such extended token sequences is a natural fit for Poolformer, capable of efficiently ingesting long token streams.

\bibliographystyle{plain}
\bibliography{bibentries}

\end{document}